\documentclass{article}

\usepackage{fullpage}
\usepackage{authblk}
\usepackage{natbib}

\usepackage{amsmath}
\usepackage{amssymb}
\usepackage{amsthm}
\usepackage{graphicx}

\usepackage{algorithm}
\usepackage{algorithmic}

\usepackage{array}
\usepackage{booktabs}
\usepackage{multirow}
\usepackage{tabularx}
\usepackage{makecell}
\usepackage{microtype}
\usepackage{enumitem}

\usepackage[hidelinks]{hyperref}

\newtheorem{theorem}{Theorem}
\newtheorem{lemma}[theorem]{Lemma}
\newtheorem{proposition}[theorem]{Proposition}
\newtheorem{corollary}[theorem]{Corollary}
\newtheorem{assumption}[theorem]{Assumption}
\theoremstyle{definition}
\newtheorem{definition}[theorem]{Definition}
\theoremstyle{remark}
\newtheorem{remark}[theorem]{Remark}

\newenvironment{keywords}
{\par\noindent\textbf{Keywords: }\ignorespaces}
{\par}

\title{Stability and Generalization of Stochastic Nonconvex \\ Optimization under Heavy-Tailed Noise: \\From Gradient Norm to Goldstein Stationarity\thanks{A preliminary version of this work for smooth nonconvex problems appeared at ICML 2026 \citep{chen2026stability}. This paper extends our conference version to establish population Goldstein stationarity guarantees for nonsmooth nonconvex problems.}}

\author[1]{Hongxu Chen}
\author[2]{Ke Wei}
\author[1]{Xiaoming Yuan}
\author[2]{Luo Luo}
\affil[1]{Department of Mathematics, The University of Hong Kong, Hong Kong, China}
\affil[2]{School of Data Science, Fudan University, Shanghai, China.}

\date{\today}

\begin{document}

\maketitle

\begin{abstract}
The empirical evidence indicates that stochastic optimization with heavy-tailed gradient noise is more appropriate to characterize the training of machine learning models than that with standard bounded gradient variance noise. 
Most existing works on this phenomenon focus on the convergence of optimization errors, while the analysis for generalization bounds under the heavy-tailed gradient noise remains limited. 
In this paper, we develop stability-based frameworks to establish dimension-independent generalization bounds of nonconvex optimization under heavy-tailed noise for both smooth and nonsmooth settings. 
Specifically, we introduce a truncation argument to achieve the generalization error bound based on the algorithmic stability under finite $p$th moment conditions with $p\in(1,2]$.
For the smooth case, we apply our framework to derive  population gradient norm bounds for several popular stochastic optimization algorithms, including clipped and normalized stochastic gradient descent, as well as their mini-batch and momentum variants.
For the nonsmooth case, we establish a population Goldstein stationarity guarantee for online-to-nonconvex conversion methods.
It is worth noting that our theory for the nonsmooth setting is established by introducing an on-average gradient stability for the smooth surrogate function, which does not require the additional weak convexity assumption used in existing analysis based on the Moreau envelope.
\end{abstract}


\section{Introduction}

Modern machine learning methods have achieved remarkable success in areas such as computer vision \citep{voulodimos2018deep,szeliski2022computer} and natural language processing \citep{devlin2019bert,chowdhary2020natural}. 
Many machine learning problems can be formulated as minimizing the expected loss over a data distribution. 
Let $\xi$ be a random sample drawn from a distribution $\mathcal{D}$ on a sample space $\Xi$, and $f(x;\xi)$ denote the loss of parameter $x$ on sample $\xi$. 
The corresponding population risk minimization problem is
\begin{equation}\label{eq_prm}
\min_{x\in\mathbb{R}^d} F(x):=\mathbb{E}_{\xi\sim\mathcal{D}}\bigl[f(x;\xi)\bigr].
\end{equation}
Since the distribution $\mathcal{D}$ is typically unknown, the population risk cannot be evaluated directly. 
It is therefore commonly approximated by the empirical risk computed from a training dataset 
$S=\{\xi_1,\ldots,\xi_n\}$ consisting of $n$ i.i.d.\ samples drawn from $\mathcal{D}$, resulting in the empirical risk minimization problem  
\begin{equation}\label{eq_erm}
\min_{x\in\mathbb{R}^d} F_S(x):=\frac{1}{n}\sum_{i=1}^n f(x;\xi_i).
\end{equation}

For large-scale applications, computing the full gradient at every iteration can be prohibitively expensive.
Therefore, stochastic optimization methods form each update from a single random sample or a small mini-batch of samples rather than from the full dataset.
For smooth problems, a representative method is stochastic gradient descent (SGD), which uses the gradient evaluated on sampled data to approximate the full gradient \citep{robbins1951stochastic,bottou2018optimization}. 
Its low per-iteration cost and simple update rule make it effective for training large models. 
Building on this basic scheme, numerous variants have been developed, including momentum and acceleration \citep{polyak1964some,nesterov1983method,sutskever2013importance,chen2024convergence}, variance reduction \citep{johnson2013accelerating,defazio2014saga,fang2018spider,li2021page}, adaptive learning rate methods \citep{duchi2011adaptive,kingma2014adam,reddi2018convergence}, and stochastic quasi-Newton methods \citep{wang2017stochastic}. 
For nonsmooth problems, stochastic subgradient methods extend SGD by using the subgradients of the sampled loss \citep{nedic2001incremental,shamir2013stochastic}. 
Other widely studied approaches include stochastic proximal methods for composite objectives \citep{duchi2010composite,xiao2010dual}, randomized smoothing \citep{duchi2012randomized}, and stochastic zeroth-order methods \citep{flaxman2005online,duchi2015optimal,nesterov2017random}.

While stochastic optimization algorithms solve the empirical risk minimization problem  efficiently, the goal of machine learning is to achieve good performance on unseen data, as measured by the population risk. 
Generalization analysis thus concerns the discrepancy between these two risks at the output of an algorithm. 
Two main approaches have been developed: uniform convergence and algorithmic stability. 
Uniform convergence controls this discrepancy simultaneously over the entire hypothesis class. 
Its bounds therefore depend on the complexity of the class and often on the model dimension \citep{bartlett2002rademacher,shalev2010learnability}. 
Algorithmic stability, in contrast, measures how the algorithm output changes when a single training sample is replaced. 
By focusing on a particular training algorithm, stability can exploit its update structure and often yield dimension-independent bounds, making it well suited to high-dimensional stochastic optimization. 
The seminal work of \citet{bousquet2002stability} introduced uniform stability and established its connection to generalization. 
Building on this framework, \citet{hardt2016train} analyzed SGD with smooth losses in both convex and nonconvex settings. 
\cite{feldman2019high} established nearly optimal high-probability bounds for uniformly stable algorithms. 
\cite{bassily2020stability} extended the stability analysis of SGD to nonsmooth convex losses. 
More recently, \citet{lei2023stability} studied stability and generalization for smooth nonconvex and weakly convex nonsmooth problems.

However, existing generalization analyses typically rely on the standard bounded-variance assumption. 
Recent empirical studies suggest that stochastic gradient noise can exhibit heavy-tailed behavior in machine learning tasks such as training deep neural networks \citep{simsekli2019tail,zhang2020adaptive,battash2024revisiting}. 
In stochastic optimization, heavy-tailed noise is commonly characterized by assuming a bounded $p$th moment for some $p\in(1,2]$. 
Note that the case of $p<2$ allows the noise to have infinite variance and is therefore more general than the bounded-variance assumption. 
Although convergence under this condition has been extensively studied, the corresponding generalization theory remains limited. 
Existing work has studied the generalization of SGD by modeling its dynamics as stochastic differential equations driven by $\alpha$-stable L\'evy noise and deriving Wasserstein stability bounds \citep{raj2023algorithmic,raj2023general}. 
More recently, generalization guarantees for clipped SGD have been established under regularly varying heavy-tailed data or bounded $p$th centered moments, but only for smooth convex problems \citep{shestakov2026optimization,zeng2026stochastic}. A unified stability and generalization theory for nonconvex and nonsmooth problems under heavy-tailed noise is still lacking.

In this paper, we study the generalization of stochastic nonconvex optimization under heavy-tailed noise and develop stability-based frameworks for both smooth and nonsmooth problems. Table~\ref{tab:generalization_comparison} compares our generalization bounds with existing results.
Our main contributions are summarized as follows.
\begin{itemize}[leftmargin=0.6cm,topsep=0.00cm,itemsep=0.00cm]
\item We propose a novel truncation technique for generalization analysis under heavy-tailed noise. 
It decomposes the stochastic gradient noise into a truncated component and a residual tail. 
This decomposition forms the basis of our generalization frameworks for both smooth and nonsmooth settings. 
The resulting guarantees fill the gap in the generalization theory of stochastic optimization under weak moment conditions.
\item For smooth nonconvex problems, we establish a generalization bound based on algorithmic stability under the bounded $p$th centered-moment condition, where $p\in(1,2]$. 
Our result extends the bounded variance analysis of \citet{lei2023stability} to heavy-tailed settings in which stochastic gradients may have infinite variance. 
We further apply this framework to clipped SGD and normalized SGD, including their mini-batch and momentum variants, and derive the corresponding population gradient bounds.
\item For nonsmooth nonconvex problems, we develop a generalization framework for Goldstein stationarity by introducing on-average gradient stability. 
Under a finite $p$th moment condition on the random Lipschitz constants, the framework provides population Goldstein stationarity guarantees without the weak convexity assumption \citep{davis2022graphical,lei2023stability}.
To the best of our knowledge, such a dimension-independent generalization result has not been established even in the bounded variance case where $p=2$. 
We further apply the framework to an online-to-nonconvex method and establish its population Goldstein stationarity guarantee.
\end{itemize}

\begin{table}[t]
\centering
\caption{We compare our generalization errors with related work for nonconvex problems, including both smooth and nonsmooth settings.
Here, we use $\epsilon$, $\epsilon_a$, and $\epsilon_\delta$ to denote uniform gradient stability, uniform argument stability, and on-average gradient stability, respectively.
Note that our bounds hold under finite $p$th moment conditions with~$p\in(1,2]$.}
\label{tab:generalization_comparison}
\vspace{4pt}

\small
\setlength{\tabcolsep}{6pt}
\renewcommand{\arraystretch}{1.35}

\begin{tabular}{@{}cccc@{}}
\toprule
Problem & Generalization measure & Bound & Assumptions \\
\midrule

\multirow{3}{*}[-19pt]{Smooth}
& \multirow{3}{*}[-19pt]{Gradient norm}
& \makecell[c]{
    $O\big(\sqrt{d\log(n)/n} \, \big)$ \\[2pt]
    {\scriptsize \citep{mei2018landscape}}
  }
& Sub-Gaussian \\
\cmidrule{3-4}

&
& \makecell[c]{
    $O\big(\epsilon+n^{-\frac12}\big)$ \\[2pt]
    {\scriptsize \citep{lei2023stability}}
  }
& Bounded variance \\
\cmidrule{3-4}

&
& \makecell[c]{
    $O\Big(\epsilon+n^{-\frac{p-1}{p}}\Big)$ \\[2pt]
    {\scriptsize\textbf{Theorem~\ref{thm:smooth_generalization}}}
  }
& \ref{eq:p-bcm} \\
\midrule

\multirow{4}{*}[-28pt]{Nonsmooth}
& \multirow{2}{*}[-11pt]{
    \makecell[c]{Gradient norm of\\[2pt]Moreau envelope}
  }
& \makecell[c]{
    $O\big(\sqrt{d\log(n)/n} \, \big)$ \\[2pt]
    {\scriptsize \citep{davis2022graphical}}
  }
& \makecell[c]{
    Weakly convex \\[2pt]
    Mean-squared Lipschitz\textsuperscript{\dag}
  } \\
\cmidrule{3-4}

&
& \makecell[c]{
    $O\Big(\epsilon_a^{\frac12}+n^{-\frac12}\Big)$ \\[2pt]
    {\scriptsize \citep{lei2023stability}}
  }
& \makecell[c]{
    Weakly convex \\[2pt]
    Uniform Lipschitz\textsuperscript{\ddag}
  } \\
\cmidrule{2-4}

& \multirow{2}{*}[-11pt]{Goldstein stationarity}
& \makecell[c]{
    $O\big(\sqrt{d\log(n)/n} \, \big)$ \\[2pt]
    {\scriptsize \citep{kornowski2025improved}}
  }
& Uniform Lipschitz\textsuperscript{\ddag} \\
\cmidrule{3-4}

&
& \makecell[c]{
    $O\Big(\epsilon_\delta^{p-1}
      +n^{-\frac{p-1}{p}}\Big)$ \\[2pt]
    {\scriptsize\textbf{Theorem~\ref{thm:nonsmooth_generalization}}}
  }
& $p$-integrable Lipschitz\textsuperscript{\S} \\

\bottomrule
\end{tabular}

\par\vspace{4pt}
\parbox{\linewidth}{\footnotesize\raggedright
\textsuperscript{\dag}\,
The mean-squared Lipschitz condition means the component $f(\cdot;\xi)$ is $L(\xi)$-Lipschitz with $\mathbb{E}_{\xi}[L(\xi)^2]\le L^2$.\\[3pt]
\textsuperscript{\ddag}\,
The uniform Lipschitz condition means the  component $f(\cdot;\xi)$ is $L$-Lipschitz uniformly over $\xi$.\\[3pt]
\textsuperscript{\S}\,
The $p$-integrable Lipschitz condition means the  component $f(\cdot;\xi)$ is $L(\xi)$-Lipschitz with $\mathbb{E}_{\xi}[L(\xi)^p]\le L^p$.
}
\end{table}

\vspace{3pt}
The rest of the paper is organized as follows. In Section~\ref{sec:related_work}, we provide a literature review. In Section~\ref{sec:smooth}, we present a generalization framework for smooth nonconvex problems under heavy-tailed noise and its applications to clipped and normalized SGD. 
In Section~\ref{sec:nonsmooth}, we study nonsmooth nonconvex problems by establishing generalization guarantees for Goldstein stationarity and analyzing the online-to-nonconvex conversion method. Finally, we conclude our work in Section~\ref{sec:conclusion}.

\section{Related Work}
\label{sec:related_work}
In this section, we review the literature on algorithmic stability and generalization, stochastic optimization under heavy-tailed noise, and nonsmooth nonconvex optimization.

\subsection{Algorithmic Stability and Generalization}
\label{subsec:related_stability}

Uniform convergence bounds the generalization gap over a hypothesis class, and the resulting bounds often depend on the dimension or the complexity of the class \citep{bartlett2002rademacher,mei2018landscape,foster2018uniform}. 
Algorithmic stability instead measures how the output of a learning algorithm changes when a single training sample is replaced. 
\citet{bousquet2002stability} introduced uniform stability and established its connection to generalization error. 
Extensions include randomized learning algorithms \citep{elisseeff2005stability}, the characterization of learnability \citep{shalev2010learnability}, and generalization bounds based on argument stability \citep{liu2017algorithmic}. 
High-probability bounds for uniformly stable algorithms were established in \citep{feldman2019high,bousquet2020sharper}, with fast excess-risk bounds under a Bernstein condition in \citep{klochkov2021stability}. 
For convex stochastic optimization, existing results cover smooth losses \citep{hardt2016train}, risk-dependent bounds without uniformly bounded gradients \citep{lei2020fine,lei2021generalization}, nonsmooth losses \citep{bassily2020stability}, and biased stochastic gradient methods including clipped and zeroth-order SGD \citep{zeng2026stochastic}.

For nonconvex problems, \citet{hardt2016train} and \citet{neu2021information} established a generalization bound in function values for SGD, while \citet{kuzborskij2018data} derived data-dependent bounds involving local curvature. 
Related output-dependent bounds were obtained by \citet{lei2026convergence} in terms of the optimization error and local curvature, whereas \citet{charles2018stability} considered algorithms converging to global minimizers under the PL condition. 
\citet{lei2023stability} introduced $\epsilon$-uniform gradient stability and controlled the difference between population and empirical gradients by $O(\epsilon+n^{-1/2})$ under bounded variance.
For nonsmooth nonconvex problems, \citet{davis2022graphical} established finite-sample guarantees for graphical convergence of subgradients under weak convexity.
\citet{lei2023stability} established a generalization bound for weakly convex objectives by controlling the difference between the gradients of the population and empirical Moreau envelopes. 
Moreover, \citet{kornowski2025improved} applied the uniform convergence result of \citet{mei2018landscape}, leading to a Goldstein stationarity bound with dependence on the dimension $d$.

Generalization under heavy-tailed or unbounded stochastic gradient noise has been studied for several specific probabilistic models recently. 
Using $\alpha$-stable L\'evy-driven stochastic differential equations, \citet{raj2023algorithmic,raj2023general} derived Wasserstein stability bounds for SGD, and \citet{dang2025algorithmic} extended this analysis to SGD with momentum. 
\citet{sun2025nonconvex} studied normalized SGD with momentum under state-dependent unbounded noise and weak gradient Lipschitz continuity. 
\citet{zeng2026stochastic} derived stability and excess-risk bounds for clipped SGD in the smooth convex setting. 
\citet{shestakov2026optimization} studied the generalization for strongly convex smooth objectives under regularly varying gradient tails and established finite-sum convergence guarantees for clipped SGD.
Generalization for nonconvex stochastic optimization under weak moment conditions remains largely unexplored.

\subsection{Stochastic Optimization under Heavy-Tailed Noise}

Clipping and normalization are widely used approaches for controlling stochastic gradients with exceptionally large norms under heavy-tailed noise.
Gradient clipping was initially used to prevent exploding gradients in recurrent neural networks \citep{pascanu2013difficulty} and later became a standard component of private learning \citep{abadi2016deep}. 
For convex problems, clipped mirror-descent methods and related approaches admit nonasymptotic guarantees under weak tail assumptions \citep{nazin2019algorithms,davis2021low,gorbunov2020stochastic}. 
For smooth nonconvex problems, \citet{zhang2020adaptive} established convergence bounds and matching lower bounds for clipped SGD, while subsequent works provided high-probability guarantees under unbounded variance \citep{sadiev2023high,nguyen2023improved}. 
Clipping has also been combined with momentum and gradient normalization \citep{cutkosky2021high,liu2023breaking}. 
\citet{mai2021stability} analyzed clipped SGD for nonsmooth convex losses, and \citet{koloskova2023revisiting} established convergence with a constant clipping threshold under bounded variance.

Gradient normalization was first studied for deterministic convex optimization by \cite{nesterov2004lectures}. 
\citet{hazan2015beyond} introduced normalized gradient descent in stochastic optimization under quasi-convexity and local Lipschitz assumptions, and \citet{levy2017online} analyzed an adaptive variant for smooth objectives. 
For general smooth nonconvex problems, \citet{cutkosky2020momentum} showed that momentum can replace the use of large batches, while \citet{yang2023two} established lower bounds and studied adaptation without prior knowledge of problem parameters. 
Nearly optimal rates under relaxed smoothness were obtained by \citet{hubler2024parameter}. 
Under heavy-tailed noise, \citet{cutkosky2021high} combined normalization with clipping and momentum to obtain high-probability guarantees, whereas \citet{liu2025nonconvex} established an optimal rate for batched normalized SGD with momentum without clipping. 
\citet{hubler2025gradient} derived complexity bounds and matching lower bounds for normalized SGD with mini-batches or momentum. 
\citet{fradin2025tight} established tight lower bounds and matching algorithms under relaxed mean-squared smoothness. 
In addition, \citet{sun2025revisiting} combined clipping and normalization with momentum to establish the sharper convergence rates by assuming the objective has Lipschitz continuous Hessian.

\subsection{Nonsmooth Nonconvex Optimization}

The Clarke subdifferential is widely used to characterize the stationarity in nonsmooth nonconvex optimization \citep{clarke1990optimization,kiwiel1986aggregate,burke2005robust,davis2020stochastic}, while the corresponding approximate stationary points cannot be found with finitely many oracle queries in general \citep{zhang2020complexity}.
Therefore, the study of its Goldstein variants  \citep{goldstein1977optimization} has attracted considerable attention in recent years.
Specifically, \citet{zhang2020complexity} established finite-time complexity guarantees for finding Goldstein stationary points in both deterministic and stochastic settings.
\citet{tian2022finite} and \citet{davis2022gradient} relaxed the oracle requirement through randomized perturbations, retaining the deterministic complexity guarantee using standard first-order oracles.
\citet{cutkosky2023optimal} further achieved the optimal oracle complexity by introducing the framework of online-to-nonconvex conversion.
For the stochastic setting with heavy-tailed noise, \citet{liu2024high} introduced gradient clipping to establish the high probability convergence guarantee.
Later, \citet{liu2026online} showed that the classical online-to-nonconvex conversion can obtain optimal convergence rates under heavy-tailed noise.
Very recently, \citet{kornowski2026complexity} studied the optimality of first-order methods by establishing tight lower complexity bounds.
Additionally, zero-order methods for finding approximate Goldstein stationary points have also been studied \citep{lin2022gradient,chen2023faster,kornowski2024algorithm,liu2026zeroth}.

\section{Generalization Bounds for Smooth Case}
\label{sec:smooth}

In this section, we study the generalization of smooth nonconvex optimization under heavy-tailed noise.
By introducing a truncation argument, we establish a generalization framework and derive population gradient bounds for clipped SGD and normalized SGD.

\subsection{The Smooth Problem Setting}
\label{subsec:smooth_setting}

For population and empirical problems \eqref{eq_prm} and \eqref{eq_erm}, 
we let $A$ be a possibly randomized training algorithm for minimizing function~$F_S(\cdot)$, and we denote its output by $A(S)$.
Throughout this section, we use $\|\cdot\|$ to denote the Euclidean norm, and $\mathbb{E}_{S,A}[\,\cdot\,]$ to denote the expectation over the training set and the randomness of the algorithm.
For a given initial point $x_0\in\mathbb{R}^d$, we set
\begin{align*}
F_S^* = \inf_{x\in\mathbb{R}^d}F_S(x)
\qquad \text{and} \qquad
\Delta = \mathbb{E}_{S,A}\bigl[F_S(x_0)-F_S^*\bigr] < +\infty.
\end{align*}

Since finding a global minimizer is NP-hard in general for nonconvex smooth optimization \citep{pardalos1991quadratic}, we typically target to find a stationary point.
Accordingly, the population gradient norm at the algorithm output can be decomposed as
\begin{align}
\label{eq:err_decomp}
\mathbb{E}_{S,A}\bigl[\|\nabla F(A(S))\|\bigr]
\leq
\underbrace{\mathbb{E}_{S,A}\bigl[\|\nabla F_S(A(S))\|\bigr]}_{\text{optimization error}}
+
\underbrace{\mathbb{E}_{S,A}\bigl[\|\nabla F(A(S))-\nabla F_S(A(S))\|\bigr]}_{\text{generalization error}}. 
\end{align}
Note that the generalization error will be bounded through algorithmic stability, while the optimization error will be analyzed for specific stochastic algorithms.

We then introduce assumptions used in our analysis for the smooth case.
Specifically, we suppose each stochastic component is smooth, which is
standard in stochastic smooth nonconvex optimization and will be used in the stability analysis of clipped and normalized SGD methods.

\begin{assumption}
\label{ass_smooth}
For all $\xi\in\Xi$, the stochastic component $f(\cdot;\xi)$ is $L$-smooth.
That is, for all $x,y\in\mathbb{R}^d$, it holds
\begin{align*}
\|\nabla f(x;\xi)-\nabla f(y;\xi)\| &\leq L\|x-y\|.
\end{align*}
\end{assumption}
We characterized the heavy-tailed noise of the stochastic gradient by the following bounded $p$th centered moment condition.

\begin{assumption}
\label{ass_pbcm}
There exist some $p\in(1,2]$ and $\sigma_p>0$ such that for all $x\in\mathbb{R}^d$, it holds
\begin{align}
\mathbb{E}_{\xi\sim\mathcal{D}}
\bigl[\|\nabla f(x;\xi)-\nabla F(x)\|^p\bigr]
&\leq \sigma_p^p.
\tag{$p$-BCM}\label{eq:p-bcm}
\end{align}
\end{assumption}

\begin{remark}
Assumption~\ref{ass_pbcm} is widely used in stochastic optimization under heavy-tailed noise
\citep{zhang2020adaptive,cutkosky2021high,nguyen2023improved,liu2023breaking,hubler2025gradient}.
Another way to model heavy-tailed noise is to assume an $\alpha$-stable random variable $X$.
Its characteristic function is
$\mathbb{E}[\exp(i\omega X)]=\exp(-\sigma|\omega|^\alpha)$,
where $i$ is the imaginary unit and $\omega\in\mathbb{R}$ \citep{simsekli2019tail}.
However, the $\alpha$-stable model can be restrictive because it implicitly assumes that the noise is identically distributed across coordinates \citep{li2021validity,xie2021diffusion}, while our Assumption \ref{ass_pbcm} imposes a weaker moment condition and is therefore more general.
Additionally, recent work \citep{sun2025nonconvex} introduced the condition $\mathbb{E}_{\xi \sim \mathcal{D}}\bigl[\|\nabla f(x;\xi)-\nabla F(x)\|^2\bigr] \le c\|x\|^p+\sigma^2$ for some $c>0$ to relax the bounded variance condition. 
Note that this condition and our Assumption~\ref{ass_pbcm} cannot be derived from each other.
\end{remark}

\subsection{Generalization Error for the Smooth Case}

We target to establish a generalization framework based on the algorithmic stability, which measures the sensitivity of a training algorithm to the replacement of a single sample.
We call two datasets $S$ and $\widetilde{S}$ neighboring datasets if they differ by one sample, and define the argument stability as follows.

\begin{definition}
\label{def:argument_stability}
A randomized algorithm $A$ is $\epsilon_a$-uniformly argument stable if for all neighboring datasets $S$ and $\widetilde{S}$, it holds
\begin{align*}
\mathbb{E}_{A}\bigl[\|A(S)-A(\widetilde{S})\|\bigr]\le \epsilon_a.
\end{align*}
\end{definition}
For smooth nonconvex problems, it is natural to consider the following gradient stability.
\begin{definition}
\label{def:gradient_stability}
A randomized algorithm $A$ is $\epsilon$-uniformly gradient stable if for all neighboring datasets $S$ and $\widetilde{S}$, it holds
\begin{align*}
\sup_{\xi\in\Xi}\mathbb{E}_{A}\bigl[\|\nabla f(A(S);\xi)-\nabla f(A(\widetilde{S});\xi)\|^2\bigr] &\leq \epsilon^2.
\end{align*}
\end{definition}

\begin{remark}
In the bounded variance setting (Assumption \ref{ass_pbcm} with $p=2$),
\citet[Theorem~6]{lei2023stability} showed that the $\epsilon$-uniform gradient stability implies the generalization error bound
\begin{align*}
\mathbb{E}_{S,A}\bigl[\|\nabla F(A(S))-\nabla F_S(A(S))\|\bigr] &\leq 4\epsilon+\sigma_2n^{-\frac{1}{2}},
\end{align*}
which is subsequently used to analyze several popular stochastic nonconvex optimization methods such as SGD, SignSGD, and AdaGrad \citep{lei2023stability,chen2024three,zhang2024generalization,sun2025investigating}.
However, these results do not include our $p$-BCM condition (Assumption \ref{ass_pbcm}) when $p\in(1,2)$.    
\end{remark}

For heavy-tailed gradient noise, we establish the following generalization error bound under the $p$-BCM condition by introducing a truncation technique.

\begin{theorem}
\label{thm:smooth_generalization}
Suppose that the randomized algorithm $A$ is $\epsilon$-uniformly gradient stable and Assumption~\ref{ass_pbcm} holds. Then we have
\begin{align*}
\mathbb{E}_{S,A}\bigl[\|\nabla F(A(S))-\nabla F_S(A(S))\|\bigr] &\leq 4\epsilon+C_p\sigma_p n^{-\frac{p-1}{p}},
\end{align*}
where $C_p>0$ only depends on $p$ such that $C_p=\frac{p}{2(p-1)}\left(\frac{4(p-1)}{2-p}\right)^{\frac{2-p}{p}}$ for $p\in(1,2)$ and $C_2=1$.
\end{theorem}

\begin{proof}
Let $S'=\{\xi_1',\ldots,\xi_n'\}$ such that all of its elements are drawn independently from distribution $\mathcal{D}$. 
For each $i\in[n]$, we define
\begin{align*}
S^{(i)}=\{\xi_1,\ldots,\xi_{i-1},\xi_i',\xi_{i+1},\ldots,\xi_n\}.
\end{align*}
We then establish the generalization error bound according to the following four steps.

\noindent\textbf{Step 1: Decomposition and truncation.}

Following \citet{lei2023stability}, the gradient discrepancy can be decomposed as
\begin{align}\label{eq:smooth_ghost_decomposition}
\begin{split}    
& n\bigl(\nabla F(A(S))-\nabla F_S(A(S))\bigr) \\
= & \sum_{i=1}^n\mathbb{E}_{\xi,\xi_i'}\bigl[\nabla f(A(S);\xi)-\nabla f(A(S^{(i)});\xi)\bigr] \\
& +\sum_{i=1}^n\mathbb{E}_{\xi_i'}\bigl[\mathbb{E}_{\xi}[\nabla f(A(S^{(i)});\xi)]-\nabla f(A(S^{(i)});\xi_i)\bigr] \\
& +\sum_{i=1}^n\mathbb{E}_{\xi_i'}\bigl[\nabla f(A(S^{(i)});\xi_i)-\nabla f(A(S);\xi_i)\bigr].
\end{split}
\end{align}
We define
\begin{align*}
z_i(S) &= \mathbb{E}_{\xi_i'}\bigl[\mathbb{E}_{\xi}[\nabla f(A(S^{(i)});\xi)]-\nabla f(A(S^{(i)});\xi_i)\bigr].
\end{align*}
Taking norms and then expectations in equation \eqref{eq:smooth_ghost_decomposition}, we obtain
\begin{align}
& n\mathbb{E}_{S,A}\bigl[\|\nabla F(A(S))-\nabla F_S(A(S))\|\bigr] \notag\\
&\leq \sum_{i=1}^n\mathbb{E}_{S,A,\xi,\xi_i'}\bigl[\|\nabla f(A(S);\xi)-\nabla f(A(S^{(i)});\xi)\|\bigr] + \mathbb{E}_{S,A}\left[\left\|\sum_{i=1}^n z_i(S)\right\|\right] \notag\\
& \quad + \sum_{i=1}^n\mathbb{E}_{S,A,\xi_i'}\bigl[\|\nabla f(A(S^{(i)});\xi_i)-\nabla f(A(S);\xi_i)\|\bigr] \notag\\
&\leq 2n\epsilon+\mathbb{E}_{S,A}\left[\left\|\sum_{i=1}^n z_i(S)\right\|\right],
\label{eq:smooth_z_reduction}
\end{align}
where the first step holds by the triangle inequality, and the second step is based on $\epsilon$-uniform gradient stability. 
It remains to control the second term of equation \eqref{eq:smooth_z_reduction}.
For given $\tau>0$, we define the clipping operator by
\begin{align*}
\mathrm{clip}_{\tau}(u) &=
\begin{cases}
u, & \|u\|\leq\tau,\\
\displaystyle \tau\cdot\frac{u}{\|u\|}, & \|u\|>\tau.
\end{cases}
\end{align*}
For all $p\in(1,2]$ and $u,v\in\mathbb{R}^d$, one has
\begin{equation}
\label{eq:clipping_properties}
\begin{aligned}
\|\mathrm{clip}_{\tau}(u)\|^2 &\leq \tau^{2-p}\|u\|^p,\\
\|u-\mathrm{clip}_{\tau}(u)\| &\leq \tau^{1-p}\|u\|^p,\\
\|\mathrm{clip}_{\tau}(u)-\mathrm{clip}_{\tau}(v)\| &\leq \|u-v\|,
\end{aligned}
\end{equation}
where the first two inequalities directly follow the definition of clipping operator and the last inequality holds because $\mathrm{clip}_{\tau}(\cdot)$ is a Euclidean projection onto the closed ball of radius $\tau$.

We construct a centered truncation by defining
\begin{align*}
T_{\tau}(x;\xi) &= \mathrm{clip}_{\tau}\bigl(\nabla f(x;\xi)-\nabla F(x)\bigr),\\
M_{\tau}(x) &= \mathbb{E}_{\xi}\bigl[T_{\tau}(x;\xi)\bigr],\\
\widetilde{T}_{\tau}(x;\xi) &= T_{\tau}(x;\xi)-M_{\tau}(x).
\end{align*}
For all $i\in[n]$, we define
\begin{align*}
z_i^{(\tau)}(S) &= -\mathbb{E}_{\xi_i'}\bigl[\widetilde{T}_{\tau}(A(S^{(i)});\xi_i)\bigr]
\qquad \text{and} \qquad
r_i^{(\tau)}(S)=z_i(S)-z_i^{(\tau)}(S).
\end{align*}
By the definition of $r_i^{(\tau)}(S)$ and the triangle inequality, we have
\begin{align}
\label{eq_smooth_z_and_r}
\left\|\sum_{i=1}^n z_i(S)\right\|
&\leq \left\|\sum_{i=1}^n z_i^{(\tau)}(S)\right\|
+\left\|\sum_{i=1}^n r_i^{(\tau)}(S)\right\|.
\end{align}

\noindent\textbf{Step 2: Bound the term $\mathbb{E}_{S,A}\big[\|\sum_{i=1}^n z_i^{(\tau)}(S)\|\big]$.}

Based on the definition of $z_i^{(\tau)}(S)$, it holds
$\mathbb{E}_{\xi_i}\bigl[z_i^{(\tau)}(S)\bigr]=0$ and
\begin{align}\label{eq:ztS}
\mathbb{E}_{S,A}\left[\left\|\sum_{i=1}^n z_i^{(\tau)}(S)\right\|^2\right]
&= \sum_{i=1}^n\mathbb{E}_{S,A}\bigl[\|z_i^{(\tau)}(S)\|^2\bigr]
+\sum_{i\neq j}\mathbb{E}_{S,A}\bigl[\langle z_i^{(\tau)}(S),z_j^{(\tau)}(S)\rangle\bigr].
\end{align}
We first consider the first term on the right-hand side of equation \eqref{eq:ztS}.
By Jensen's inequality and the fact
$\mathbb{E}\|X-\mathbb{E}X\|^2\leq\mathbb{E}\|X\|^2$, we have
\begin{align}
\label{eq_smooth_dia1}
\mathbb{E}_{S,A}\bigl[\|z_i^{(\tau)}(S)\|^2\bigr]
&\leq \mathbb{E}_{S,A,\xi_i'}\bigl[\|T_{\tau}(A(S^{(i)});\xi_i)\|^2\bigr].
\end{align}
The first inequality in equation \eqref{eq:clipping_properties} shows that
\begin{align}
\label{eq_smooth_dia2}
\|T_{\tau}(A(S^{(i)});\xi_i)\|^2
&\leq \tau^{2-p}\|\nabla f(A(S^{(i)});\xi_i)-\nabla F(A(S^{(i)}))\|^p.
\end{align}
Combining equations \eqref{eq_smooth_dia1} and \eqref{eq_smooth_dia2}, and invoking Assumption~\ref{ass_pbcm}, we obtain
\begin{align}\label{eq:smooth_diagonal_bound}
\mathbb{E}_{S,A}\bigl[\|z_i^{(\tau)}(S)\|^2\bigr]
&\leq \tau^{2-p}\sigma_p^p.
\end{align}

We then consider the second term on the right-hand side of equation \eqref{eq:ztS}.
We use the notation $S''=\{\xi_1'',\ldots,\xi_n''\}$ to denote another independent copy of $S$.
Accordingly, we define
\begin{align*}
S_j & = \{ \xi_1, \dots, \xi_{j-1}, \xi_j'', \xi_{j+1}, \dots, \xi_n \}, \\
S_j^{(i)} & = \{ \xi_1, \dots, \xi_{i-1}, \xi_i', \xi_{i+1}, \dots, \xi_{j-1}, \xi_j'', \xi_{j+1}, \dots, \xi_n \}.
\end{align*}
In the case of $i\neq j$, the centering of $\widetilde{T}_{\tau}$ ensures
\begin{align*}
\mathbb{E}_{S,A}\bigl[\langle z_i^{(\tau)}(S_j),z_j^{(\tau)}(S)\rangle\bigr] &= 0,\\
\mathbb{E}_{S,A}\bigl[\langle z_i^{(\tau)}(S),z_j^{(\tau)}(S_i)\rangle\bigr] &= 0,\\
\mathbb{E}_{S,A}\bigl[\langle z_i^{(\tau)}(S_j),z_j^{(\tau)}(S_i)\rangle\bigr] &= 0.
\end{align*}
Using the above equations and Young's inequality, we obtain
\begin{align}\label{eq:smooth_cross_young}
&\mathbb{E}_{S,A}\bigl[\langle z_i^{(\tau)}(S),z_j^{(\tau)}(S)\rangle\bigr] \notag\\
&= \mathbb{E}_{S,S'',A}\Bigl[\bigl\langle z_i^{(\tau)}(S)-z_i^{(\tau)}(S_j),z_j^{(\tau)}(S)-z_j^{(\tau)}(S_i)\bigr\rangle\Bigr] \notag\\
&\leq \frac{1}{2}\mathbb{E}_{S,S'',A}\bigl[\|z_i^{(\tau)}(S)-z_i^{(\tau)}(S_j)\|^2\bigr]
+\frac{1}{2}\mathbb{E}_{S,S'',A}\bigl[\|z_j^{(\tau)}(S)-z_j^{(\tau)}(S_i)\|^2\bigr].
\end{align}
By Jensen's inequality, we have
\begin{align}\label{eq:smooth_difference_jensen}
\|z_i^{(\tau)}(S)-z_i^{(\tau)}(S_j)\|^2
&\leq \mathbb{E}_{\xi_i'}\bigl[\|\widetilde{T}_{\tau}(A(S^{(i)});\xi_i)-\widetilde{T}_{\tau}(A(S_j^{(i)});\xi_i)\|^2\bigr].
\end{align}
Moreover, it holds
\begin{align}\label{eq:smooth_centering_bound}
\mathbb{E}_{\xi_i}\bigl[\|\widetilde{T}_{\tau}(A(S^{(i)});\xi_i)-\widetilde{T}_{\tau}(A(S_j^{(i)});\xi_i)\|^2\bigr]
&\leq 4\mathbb{E}_{\xi_i}\bigl[\|T_{\tau}(A(S^{(i)});\xi_i)-T_{\tau}(A(S_j^{(i)});\xi_i)\|^2\bigr],
\end{align}
where the inequality follows from the definition of $M_{\tau}$, Jensen's inequality, and the elementary inequality
$\|u+v\|^2\leq 2\|u\|^2+2\|v\|^2$.
Since the operator $\mathrm{clip}_{\tau}(\cdot)$ is $1$-Lipschitz, it follows
\begin{align*}
&\quad \|T_{\tau}(A(S^{(i)});\xi_i)-T_{\tau}(A(S_j^{(i)});\xi_i)\|^2\\
&\leq \bigl\|\bigl(\nabla f(A(S^{(i)});\xi_i)-\nabla F(A(S^{(i)}))\bigr)-\bigl(\nabla f(A(S_j^{(i)});\xi_i)-\nabla F(A(S_j^{(i)}))\bigr)\bigr\|^2.
\end{align*}
Applying the variance inequality
$\mathbb{E}\|X-\mathbb{E}X\|^2\leq\mathbb{E}\|X\|^2$ to the right-hand side, we obtain
\begin{align}\label{eq:smooth_clipping_difference}
\mathbb{E}_{\xi_i}\bigl[\|T_{\tau}(A(S^{(i)});\xi_i)-T_{\tau}(A(S_j^{(i)});\xi_i)\|^2\bigr]
&\leq \mathbb{E}_{\xi_i}\bigl[\|\nabla f(A(S^{(i)});\xi_i)-\nabla f(A(S_j^{(i)});\xi_i)\|^2\bigr].
\end{align}
Taking expectation in equation \eqref{eq:smooth_difference_jensen} and using equations \eqref{eq:smooth_centering_bound} and \eqref{eq:smooth_clipping_difference}, one has
\begin{align*}
\mathbb{E}_{S,S'',A}\bigl[\|z_i^{(\tau)}(S)-z_i^{(\tau)}(S_j)\|^2\bigr]
&\leq 4\mathbb{E}_{S,S'',A,\xi_i'}\bigl[\|\nabla f(A(S^{(i)});\xi_i)-\nabla f(A(S_j^{(i)});\xi_i)\|^2\bigr]
\leq 4\epsilon^2,
\end{align*}
where the last inequality follows from the stability condition, since $S^{(i)}$ and $S_j^{(i)}$ are neighboring datasets.
The same bound holds with $i$ and $j$ interchanged.
Substituting these two bounds into \eqref{eq:smooth_cross_young}, it holds
\begin{align}\label{eq:smooth_cross_bound}
\mathbb{E}_{S,A}\bigl[\langle z_i^{(\tau)}(S),z_j^{(\tau)}(S)\rangle\bigr]
&\leq \frac{1}{2}\cdot 4\epsilon^2+\frac{1}{2}\cdot 4\epsilon^2=4\epsilon^2.
\end{align}

Applying Jensen's inequality and using equations \eqref{eq:ztS}, \eqref{eq:smooth_diagonal_bound}, and \eqref{eq:smooth_cross_bound}, we obtain
\begin{align*}
\mathbb{E}_{S,A}\left[\left\|\sum_{i=1}^n z_i^{(\tau)}(S)\right\|\right]
&\leq \sqrt{n\tau^{2-p}\sigma_p^p+4n(n-1)\epsilon^2}.
\end{align*}
After dividing by $n$ and using $n-1\leq n$ together with
$\sqrt{a+b}\leq\sqrt{a}+\sqrt{b}$, we have
\begin{align}\label{eq:smooth_truncated_sum}
\frac{1}{n}\mathbb{E}_{S,A}\left[\left\|\sum_{i=1}^n z_i^{(\tau)}(S)\right\|\right]
&\leq \sqrt{\frac{\tau^{2-p}\sigma_p^p}{n}}+2\epsilon.
\end{align}

\noindent\textbf{Step 3: Bound the term $\mathbb{E}_{S,A}\big[\|\sum_{i=1}^n r_i^{(\tau)}(S)\|\big]$.}

From the definitions of $z_i(S)$ and $z_i^{(\tau)}(S)$, one has
\begin{align*}
r_i^{(\tau)}(S)
&= -\mathbb{E}_{\xi_i'}\Bigl[
\bigl(\nabla f(A(S^{(i)});\xi_i)-\nabla F(A(S^{(i)}))\bigr)
-T_{\tau}(A(S^{(i)});\xi_i)\Bigr]\\
&\quad-\mathbb{E}_{\xi_i'}\mathbb{E}_{\xi}\Bigl[
T_{\tau}(A(S^{(i)});\xi)
-\bigl(\nabla f(A(S^{(i)});\xi)-\nabla F(A(S^{(i)}))\bigr)\Bigr].
\end{align*}
Applying the triangle inequality and Jensen's inequality to above decomposition, we have
\begin{align*}
\mathbb{E}_{\xi_i}\bigl[\|r_i^{(\tau)}(S)\|\bigr]
&\leq 2\mathbb{E}_{\xi_i'}\mathbb{E}_{\xi}\Bigl[
\|\bigl(\nabla f(A(S^{(i)});\xi)-\nabla F(A(S^{(i)}))\bigr)
-T_{\tau}(A(S^{(i)});\xi)\|\Bigr].
\end{align*}
According to the fact
$\|u-\mathrm{clip}_{\tau}(u)\|\leq \|u\|^p/\tau^{p-1}$ shown in equation \eqref{eq:clipping_properties}
and Assumption~\ref{ass_pbcm}, we obtain
\begin{align}
\mathbb{E}_{S,A}\bigl[\|r_i^{(\tau)}(S)\|\bigr]
&\leq \frac{2\sigma_p^p}{\tau^{p-1}}.
\label{eq:smooth_residual_bound}
\end{align}

\noindent\textbf{Step 4: Choose the value of $\tau$.}

By equations \eqref{eq_smooth_z_and_r}, \eqref{eq:smooth_truncated_sum}, and \eqref{eq:smooth_residual_bound}, we have
\begin{align}\label{eq:phi-12}
\frac{1}{n}\mathbb{E}_{S,A}\left[\left\|\sum_{i=1}^n z_i(S)\right\|\right]
&\leq \sqrt{\frac{\tau^{2-p}\sigma_p^p}{n}}+2\epsilon+\frac{2\sigma_p^p}{\tau^{p-1}}.
\end{align}
In the case of $p=2$, taking $\tau \to +\infty$ yields
\begin{align}\label{eq:phi-1}
\frac{1}{n}\mathbb{E}_{S,A}\left[\left\|\sum_{i=1}^n z_i(S)\right\|\right]
&\leq \frac{\sigma_2}{\sqrt{n}}+2\epsilon.
\end{align}
In the case of $p\in(1,2)$, we define
\begin{align*}
\phi(\tau)
&= \sqrt{\frac{\sigma_p^p}{n}}\,\tau^{\frac{2-p}{2}}
+2\sigma_p^p\tau^{1-p},
\end{align*}
where $\tau > 0$. Then we have
\begin{align*}
\phi'(\tau)
&= \sqrt{\frac{\sigma_p^p}{n}}\,\frac{2-p}{2}\tau^{-\frac{p}{2}}
+2\sigma_p^p(1-p)\tau^{-p}.
\end{align*}
Note that the equation $\phi'(\tau)=0$ has the unique solution
\begin{align*}
\tau_* &= \left(\frac{4(p-1)}{2-p}\sqrt{n\sigma_p^p}\right)^{\frac{2}{p}},
\end{align*}
and it holds $\phi'(\tau)<0$ for $\tau\in(0,\tau_*)$ and $\phi'(\tau)>0$ for $\tau>\tau_*$. Hence, the point $\tau_*$ is the unique global minimizer of $\phi(\cdot)$ over $(0,\infty)$.
It follows that
\begin{align}\label{eq:phi-star}
\phi(\tau_*)
= \frac{p}{2(p-1)}
\left(\frac{4(p-1)}{2-p}\right)^{\frac{2-p}{p}}
\sigma_p n^{-\frac{p-1}{p}}
= C_p\sigma_p n^{-\frac{p-1}{p}}.
\end{align}
Combining equations \eqref{eq:phi-12}, \eqref{eq:phi-1}, and \eqref{eq:phi-star}, for every $p\in(1,2]$, it holds
\begin{align}
\frac{1}{n}\mathbb{E}_{S,A}\left[\left\|\sum_{i=1}^n z_i(S)\right\|\right]
&\leq C_p\sigma_p n^{-\frac{p-1}{p}}+2\epsilon.
\label{eq:smooth_z_final}
\end{align}
Substituting equation \eqref{eq:smooth_z_final} into equation \eqref{eq:smooth_z_reduction} completes the proof.
\end{proof}

\begin{remark} Theorem~\ref{thm:smooth_generalization} extends the result in bounded variance setting \citep{lei2023stability} to the more general heavy-tailed noise.
Specifically, it recovers the result of \citet[Theorem~6]{lei2023stability} when~$p=2$.
It is worth noting that this extension is nontrivial, with the proof hinging on a truncation argument.
In addition, the constant satisfies $C_p\in[1,3]$ and $\lim_{p\to2^-}C_p=1$.
\end{remark}
Combining Theorem~\ref{thm:smooth_generalization} with the error decomposition \eqref{eq:err_decomp} gives
\begin{align}
\mathbb{E}_{S,A}\bigl[\|\nabla F(A(S))\|\bigr]
&\leq \underbrace{\mathbb{E}_{S,A}\bigl[\|\nabla F_S(A(S))\|\bigr]}_{\text{optimization error}}
+4\epsilon+C_p\sigma_p n^{-\frac{p-1}{p}}.
\label{eq:smooth_population_framework}
\end{align}

\subsection{Applications to Clipped and Normalized SGD}

Note that our population gradient bound \eqref{eq:smooth_population_framework} provides a  framework to analyze  generalization of stochastic optimization algorithms under heavy-tailed gradient noise.
We now apply this result to study two widely used methods clipped SGD and mini-batch normalized SGD.

\subsubsection{Clipped SGD}

The stochastic gradient may exhibit exceptionally large norms when the noise is heavy-tailed.
The clipping operation can address this issue by rescaling the stochastic gradient when its norm exceeds a prescribed threshold,
yielding the clipped SGD method presented in Algorithm~\ref{alg:clipped_sgd} \citep{zhang2020adaptive}.
We establish the stability of clipped SGD as follows.

\begin{proposition}[Stability of clipped SGD]
\label{prop:clipped_stability}
Suppose that Assumption~\ref{ass_smooth} holds and the initial point is $x_0=0$. Then Algorithm~\ref{alg:clipped_sgd} is $\epsilon_{\mathrm{c}}$-uniformly gradient stable with
\begin{align*}
\epsilon_{\mathrm{c}} \leq 2L\gamma\sqrt{\frac{T}{n}}\,\sum_{t=0}^{T-1}\eta_t.
\end{align*}
\end{proposition}

\begin{proof}
Without loss of generality, suppose that neighboring datasets $S$ and $\widetilde{S}$ differ only in their $n$th samples.
We let $I(A)=\{i_t\}_{t=0}^{T-1}$ denote the sampled index sequence.
We run Algorithm~\ref{alg:clipped_sgd} on datasets $S$ and $\widetilde{S}$ using the same indices $I(A)$ and $R$.
Accordingly, we use $\{x_t\}_{t=0}^T$ and $\{\widetilde x_t\}_{t=0}^T$ to denote the corresponding iterates. The outputs are given by~$A(S)=x_R$ and~$A(\widetilde S)=\widetilde x_R$, respectively.

By the update rule and the definition of clipping operator, it holds
\begin{align*}
\|x_{t+1}-x_t\| &\leq \gamma\eta_t
\qquad \text{and} \qquad
\|\widetilde x_{t+1}-\widetilde x_t\| \leq \gamma\eta_t.
\end{align*}
The initialization $x_0=\widetilde x_0=0$ implies the outputs satisfy
\begin{align*}
\|A(S)\| &\leq \gamma\sum_{t=0}^{T-1}\eta_t
\quad \text{and} \quad
\|A(\widetilde S)\| \leq \gamma\sum_{t=0}^{T-1}\eta_t.
\end{align*}
On the event $\{n\notin I(A)\}$, the two iterates $\{x_t\}_{t=0}^T$ and $\{\widetilde x_t\}_{t=0}^T$ are identical, which implies that $A(S)=A(\widetilde S)$.
Hence, it holds
\begin{align*}
\mathbb{E}_A\bigl[\|A(S)-A(\widetilde S)\|^2\bigr]
&\leq 4\gamma^2\left(\sum_{t=0}^{T-1}\eta_t\right)^2\mathbb{P}\bigl(n\in I(A)\bigr)
\leq 4\gamma^2\left(\sum_{t=0}^{T-1}\eta_t\right)^2\frac{T}{n},
\end{align*}
where the second inequality follows from the union bound.
Assumption~\ref{ass_smooth} then implies that for every $\xi\in\Xi$, it holds
\begin{align*}
\mathbb{E}_A\bigl[\|\nabla f(A(S);\xi)-\nabla f(A(\widetilde S);\xi)\|^2\bigr]
&\leq L^2\mathbb{E}_A\bigl[\|A(S)-A(\widetilde S)\|^2\bigr]
\leq 4L^2\gamma^2\left(\sum_{t=0}^{T-1}\eta_t\right)^2\frac{T}{n}.
\end{align*}
According to Definition~\ref{def:gradient_stability}, we have finished the proof.
\end{proof}

\begin{algorithm}[t]
\caption{Clipped SGD}
\label{alg:clipped_sgd}
\begin{algorithmic}[1]
\STATE \textbf{Input:} Initial point $x_0$, stepsizes $\{\eta_t\}_{t=0}^{T-1}$, clipping threshold $\gamma>0$, iteration number $T$
\FOR{$t=0,1,\ldots,T-1$}
    \STATE Sample $i_t$ uniformly from $[n]$ \\[0.1cm]
    \STATE $x_{t+1}=x_t-\eta_t\mathrm{clip}_{\gamma}\bigl(\nabla f(x_t;\xi_{i_t})\bigr)$ \\[0.1cm]
\ENDFOR
\STATE \textbf{Output:} $x_R$, where $R$ is uniformly sampled from $\{0, 1, \dots, T-1\}$
\end{algorithmic}
\end{algorithm}

\begin{remark}
We introduce the condition $x_0=0$ only to simplify the presentation, as it is common in stability analyses \citep{lei2023stability,schliserman2022stability}.
Our main results still hold without this condition.
\end{remark}

Combining the above stability bound with the analysis of optimization error gives the following population gradient norm bound.

\begin{theorem}
\label{thm:clipped_population}
Consider Algorithm~\ref{alg:clipped_sgd} with stepsize $\eta_t=\eta>0$, clipping threshold~$\gamma>0$, and the initial point~$x_0=0$.
Suppose Assumptions \ref{ass_smooth} and \ref{ass_pbcm} hold and there exists $G>0$ such that~$\mathbb{E}_{i_t}\bigl[\|\nabla f(x_t;\xi_{i_t})\|^{p}\bigr]\le G^{p}$ for all $t$.
Then it holds
\begin{align*}
\mathbb{E}_{S,A}\bigl[\|\nabla F(A(S))\|\bigr]
&\leq \sqrt{\frac{2\Delta}{\eta T}}+G^p\gamma^{1-p}+\sqrt{L\eta G^p\gamma^{2-p}} 
+8L\gamma\eta T\sqrt{\frac{T}{n}}+C_p\sigma_p n^{-\frac{p-1}{p}},
\end{align*}
where $A(S)$ is the output of Algorithm~\ref{alg:clipped_sgd} and
$C_p>0$ is the constant defined in Theorem~\ref{thm:smooth_generalization}.
\end{theorem}

\begin{proof}
Combining Proposition~\ref{prop:clipped_stability} with equation \eqref{eq:smooth_population_framework}, we have
\begin{align}\label{eq:clipped_population_reduction}
\mathbb{E}_{S,A}\bigl[\|\nabla F(A(S))\|\bigr]
 \leq \mathbb{E}_{S,A}\bigl[\|\nabla F_S(A(S))\|\bigr]
+8L\gamma\eta T\sqrt{\frac{T}{n}}
+C_p\sigma_p n^{-\frac{p-1}{p}}.
\end{align}
It remains to bound the empirical optimization error $\mathbb{E}_{S,A}\bigl[\|\nabla F_S(A(S))\|\bigr]$.

We let $\mathcal{F}_t$ be the sigma-field generated by $S$, $x_0$ and $i_0,\ldots,i_{t-1}$.
Set
\begin{align*}
g_t &= \mathrm{clip}_{\gamma}\bigl(\nabla f(x_t;\xi_{i_t})\bigr).
\end{align*}
By the $L$-smoothness of $F_S(\cdot)$, one has
\begin{align*}
F_S(x_{t+1})
&\leq F_S(x_t)-\eta\langle\nabla F_S(x_t),g_t\rangle+\frac{L\eta^2}{2}\|g_t\|^2.
\end{align*}
Taking the conditional expectation with respect to $\mathcal{F}_t$ in the above equation and applying Young's inequality, we obtain
\begin{align}\label{eq:clipped_descent}
\begin{split}
\mathbb{E}\bigl[F_S(x_{t+1})\mid\mathcal{F}_t\bigr]
&\leq F_S(x_t)-\frac{\eta}{2}\|\nabla F_S(x_t)\|^2\\
&\quad+\frac{\eta}{2}\left\|\mathbb{E}[g_t\mid\mathcal{F}_t]-\nabla F_S(x_t)\right\|^2
+\frac{L\eta^2}{2}\mathbb{E}\bigl[\|g_t\|^2\mid\mathcal{F}_t\bigr].
\end{split}
\end{align}
By the first inequality in equation \eqref{eq:clipping_properties} and the moment assumption, it holds
\begin{align}\label{eq:clipped_second_moment}
\mathbb{E}\bigl[\|g_t\|^2\mid\mathcal{F}_t\bigr]
&\leq \gamma^{2-p}\mathbb{E}\bigl[\|\nabla f(x_t;\xi_{i_t})\|^p\mid\mathcal{F}_t\bigr]
\leq \gamma^{2-p}G^p.
\end{align}
Since $i_t$ is sampled uniformly from $[n]$, the clipping bias satisfies
\begin{align}\label{eq:clipped_bias}
\left\|\mathbb{E}[g_t\mid\mathcal{F}_t]-\nabla F_S(x_t)\right\|
&=\left\|\mathbb{E}\bigl[\mathrm{clip}_{\gamma}(\nabla f(x_t;\xi_{i_t}))-\nabla f(x_t;\xi_{i_t})\mid\mathcal{F}_t\bigr]\right\| \notag\\
&\leq \gamma^{1-p}\mathbb{E}\bigl[\|\nabla f(x_t;\xi_{i_t})\|^p\mid\mathcal{F}_t\bigr]
\leq G^p\gamma^{1-p},
\end{align}
where the inequalities follow from the second inequality in equation \eqref{eq:clipping_properties} and the moment assumption, respectively.
Substituting equations \eqref{eq:clipped_second_moment} and \eqref{eq:clipped_bias} into equation \eqref{eq:clipped_descent}, taking expectation, and summing over $t=0,\ldots,T-1$, one has
\begin{align*}
\frac{\eta}{2}\sum_{t=0}^{T-1}\mathbb{E}_{S,A}\bigl[\|\nabla F_S(x_t)\|^2\bigr]
&\leq \Delta+\frac{\eta T}{2}G^{2p}\gamma^{2-2p}+\frac{L\eta^2T}{2}G^p\gamma^{2-p},
\end{align*}
where the inequality uses the fact $F_S(x_T)\geq F_S^*$. Since the output of the algorithm is sampled uniformly from $\{x_0,\ldots,x_{T-1}\}$, the above equation implies
\begin{align*}
\mathbb{E}_{S,A}\bigl[\|\nabla F_S(A(S))\|^2\bigr]
&\leq \frac{2\Delta}{\eta T}+G^{2p}\gamma^{2-2p}+L\eta G^p\gamma^{2-p}.
\end{align*}
Applying Jensen's inequality, we obtain
\begin{align*}
\mathbb{E}_{S,A}\bigl[\|\nabla F_S(A(S))\|\bigr]
\leq \sqrt{\frac{2\Delta}{\eta T}}+G^p\gamma^{1-p}+\sqrt{L\eta G^p\gamma^{2-p}}.
\end{align*}
Substituting the above result into equation \eqref{eq:clipped_population_reduction} completes the proof.
\end{proof}

\begin{corollary}
\label{cor:clipped_rate}
Following the assumptions of Theorem~\ref{thm:clipped_population}, we take
$T \asymp n^{\frac{1}{3}}$,
$\eta\asymp n^{-\frac{p}{3(3p-2)}}$,
and~$\gamma\asymp n^{\frac{1}{3(3p-2)}}$.
Then it holds
\begin{align*}
\mathbb{E}_{S,A}\bigl[\|\nabla F(A(S))\|\bigr]
&=O\left(n^{-\frac{p-1}{3(3p-2)}}\right).
\end{align*}
\end{corollary}

\subsubsection{Normalized SGD}

Normalization is another widely used technique to handle the heavy-tailed noise in stochastic optimization, which avoids the pain of tuning the threshold in clipped SGD \citep{liu2025nonconvex,hubler2025gradient,sun2025revisiting}.
Specifically, we consider mini-batch normalized SGD (NSGD-B) presented in Algorithm~\ref{alg:nsgd_b}, where we use the convention $g_t/\|g_t\|=0$ when~$g_t=0$.
We establish the stability of NSGD-B as follows.

\begin{proposition}[Stability of NSGD-B]
\label{prop:nsgd_b_stability}
Suppose that $x_0=0$ and Assumption~\ref{ass_smooth} holds.
Then Algorithm~\ref{alg:nsgd_b} is $\epsilon_{\mathrm{b}}$-uniformly gradient stable with
\begin{align*}
\epsilon_{\mathrm{b}} \leq 2L\sqrt{\frac{BT}{n}}\,\sum_{t=0}^{T-1}\eta_t.
\end{align*}
\end{proposition}

\begin{proof}
Without loss of generality, suppose that neighboring datasets $S$ and $\widetilde S$ differ only in their $n$th samples.
We let $I(A)=\{i_t^{(b)}\}_{0\leq t\leq T-1,\,1\leq b\leq B}$ denote the sequence that consists of the indices of samples.
We run Algorithm~\ref{alg:nsgd_b} on datasets $S$ and $\widetilde S$ using the same indices $I(A)$ and $R$.
Accordingly, we use $\{x_t\}_{t=0}^T$ and $\{\widetilde x_t\}_{t=0}^T$ to denote the corresponding iterates. The outputs are given by $A(S)=x_R$ and~$A(\widetilde S)=\widetilde x_R$, respectively.

Since the algorithm iterates with normalized stochastic gradient, it holds
\begin{align*}
\|x_{t+1}-x_t\| &\leq \eta_t
\qquad \text{and} \qquad
\|\widetilde x_{t+1}-\widetilde x_t\| \leq \eta_t.
\end{align*}
The initialization $\widetilde x_0=x_0=0$ implies that the outputs satisfy
\begin{align*}
\|A(S)\| &\leq \sum_{t=0}^{T-1}\eta_t
\qquad \text{and} \qquad
\|A(\widetilde S)\| \leq \sum_{t=0}^{T-1}\eta_t.
\end{align*}
On the event $\{n\notin I(A)\}$, the two sequences $\{x_t\}_{t=0}^T$ and $\{\widetilde x_t\}_{t=0}^T$ are identical, which implies that $A(S)=A(\widetilde S)$.
Hence, it holds
\begin{align*}
\mathbb{E}_A\bigl[\|A(S)-A(\widetilde S)\|^2\bigr]
&\leq 4\left(\sum_{t=0}^{T-1}\eta_t\right)^2\mathbb{P}\bigl(n\in I(A)\bigr)
\leq 4\left(\sum_{t=0}^{T-1}\eta_t\right)^2\frac{BT}{n},
\end{align*}
where the second inequality follows from the union bound.
Assumption~\ref{ass_smooth} then implies that for all $\xi\in\Xi$, it holds
\begin{align*}
\mathbb{E}_A\bigl[\|\nabla f(A(S);\xi)-\nabla f(A(\widetilde S);\xi)\|^2\bigr]
&\leq L^2\mathbb{E}_A\bigl[\|A(S)-A(\widetilde S)\|^2\bigr]
\leq 4L^2\left(\sum_{t=0}^{T-1}\eta_t\right)^2\frac{BT}{n}.
\end{align*}
According to Definition~\ref{def:gradient_stability}, we have finished the proof. 
\end{proof}

\begin{algorithm}[t]
\caption{Mini-batch normalized SGD (NSGD-B)}
\label{alg:nsgd_b}
\begin{algorithmic}[1]
\STATE \textbf{Input:} Initial point $x_0$, stepsizes $\{\eta_t\}_{t=0}^{T-1}$, batch size $B$, iteration number $T$
\FOR{$t=0,1,\ldots,T-1$}
    \STATE Sample $i_t^{(1)},\ldots,i_t^{(B)}$ independently and uniformly from $[n]$ \\[0.1cm]
    \STATE $\displaystyle{g_t=\frac{1}{B}\sum_{b=1}^B\nabla f\big(x_t;\xi_{i_t^{(b)}}\big)}$ \\[0.1cm]
    \STATE $\displaystyle{x_{t+1}=x_t-\eta_t\cdot\frac{g_t}{\|g_t\|}}$
\ENDFOR
\STATE \textbf{Output:} $x_R$, where $R$ is uniformly sampled from $\{0, 1, \dots, T-1\}$
\end{algorithmic}
\end{algorithm}

\begin{remark}
The stability bound of NSGD-B does not depend on the magnitude of the stochastic gradients.
In contrast, for vanilla SGD with update $x_{t+1}=x_t-\eta_tg_t$, the update length is $\eta_t\|g_t\|$, which can be large under heavy-tailed noise.
\end{remark}

Combining the above stability bound with the analysis of optimization error gives the following population gradient norm bound.

\begin{theorem}
\label{thm:nsgd_b_population}
Consider Algorithm~\ref{alg:nsgd_b} with stepsize $\eta_t=\eta>0$, batch size~$B$, and the initial point~$x_0=0$.
Suppose Assumptions \ref{ass_smooth} and \ref{ass_pbcm} hold and there exists $G>0$ such that $\mathbb{E}_{i_t^{(b)}}\bigl[\|\nabla f(x_t;\xi_{i_t^{(b)}})\|^p\bigr]\leq G^p$ for all $t$ and $b$.
Then it holds
\begin{align*}
\mathbb{E}_{S,A}\bigl[\|\nabla F(A(S))\|\bigr]
&\leq \frac{\Delta}{\eta T}+\frac{L\eta}{2}+8GB^{-\frac{p-1}{p}}
+8L\eta T\sqrt{\frac{BT}{n}}+C_p\sigma_p n^{-\frac{p-1}{p}},
\end{align*}
where $A(S)$ is the output of Algorithm~\ref{alg:nsgd_b} and
$C_p>0$ is the constant defined in Theorem~\ref{thm:smooth_generalization}.
\end{theorem}

The proof of Theorem~\ref{thm:nsgd_b_population} relies on the following moment inequality for martingale difference sequences.

\begin{lemma}[\protect{\citet[Lemma~10]{hubler2025gradient}}]
\label{lem:mds_pmoment}
Let $p\in(1,2]$, and let $Y_1,\ldots,Y_m\in\mathbb{R}^d$ be a martingale difference sequence. Then we have
\begin{align*}
\mathbb{E}\left[\left\|\sum_{j=1}^mY_j\right\|^p\right]
&\leq 2\sum_{j=1}^m\mathbb{E}\bigl[\|Y_j\|^p\bigr].
\end{align*}
\end{lemma}

\begin{proof}[Proof of Theorem~\ref{thm:nsgd_b_population}]
Combining Proposition~\ref{prop:nsgd_b_stability} with equation \eqref{eq:smooth_population_framework}, we have
\begin{align}\label{eq:nsgd_b_population_reduction}
\mathbb{E}_{S,A}\bigl[\|\nabla F(A(S))\|\bigr]
&\leq \mathbb{E}_{S,A}\bigl[\|\nabla F_S(A(S))\|\bigr]
+8L\eta T\sqrt{\frac{BT}{n}}
+C_p\sigma_p n^{-\frac{p-1}{p}}.
\end{align}
It remains to bound the  optimization error $\mathbb{E}_{S,A}\bigl[\|\nabla F_S(A(S))\|\bigr]$.

Let $\mathcal{F}_t$ be the sigma-field generated by $S$ and the mini-batch indices sampled before iteration $t$.
By the $L$-smoothness of $F_S(\cdot)$, one has
\begin{align*}
F_S(x_{t+1})
&\leq F_S(x_t)-\eta\left\langle\nabla F_S(x_t),\frac{g_t}{\|g_t\|}\right\rangle+\frac{L\eta^2}{2}.
\end{align*}
For all $u,v\in\mathbb{R}^d$, Cauchy--Schwarz and triangle inequalities imply
\begin{align*}
\left\langle u,\frac{v}{\|v\|}\right\rangle
&=\left\langle u-v,\frac{v}{\|v\|}\right\rangle+\|v\|
\geq-\|u-v\|+\|v\|
\geq\|u\|-2\|u-v\|.
\end{align*}
Applying the above inequality with $u=\nabla F_S(x_t)$ and $v=g_t$, we obtain
\begin{align*}
F_S(x_{t+1})
&\leq F_S(x_t)-\eta\|\nabla F_S(x_t)\|
+2\eta\|g_t-\nabla F_S(x_t)\|
+\frac{L\eta^2}{2}.
\end{align*}
Taking expectation, summing over $t=0,\ldots,T-1$, and using $F_S(x_T)\geq F_S^*$, it holds
\begin{align}\label{eq:nsgd_b_optimization}
\frac{1}{T}\sum_{t=0}^{T-1}\mathbb{E}_{S,A}\bigl[\|\nabla F_S(x_t)\|\bigr]
&\leq \frac{\Delta}{\eta T}
+\frac{2}{T}\sum_{t=0}^{T-1}\mathbb{E}_{S,A}\bigl[\|g_t-\nabla F_S(x_t)\|\bigr]
+\frac{L\eta}{2}.
\end{align}
It remains to control the mini-batch gradient estimation error.
We let $\mathcal{G}_{t,0}=\mathcal{F}_t$ and define
\begin{align*}
X_b &= \nabla f(x_t;\xi_{i_t^{(b)}})-\nabla F_S(x_t)
\quad \text{and} \quad
\mathcal{G}_{t,b}=\sigma\bigl(\mathcal{F}_t,i_t^{(1)},\ldots,i_t^{(b)}\bigr),
\end{align*}
for each $b\in\{1,\ldots,B\}$.
Then we can write
\begin{align*}
g_t-\nabla F_S(x_t) &= \frac{1}{B}\sum_{b=1}^BX_b.
\end{align*}
Since the index $i_t^{(b)}$ is sampled independently of $\mathcal{G}_{t,b-1}$, it holds
\begin{align*}
\mathbb{E}\bigl[X_b\mid\mathcal{G}_{t,b-1}\bigr]
&=\mathbb{E}\bigl[\nabla f(x_t;\xi_{i_t^{(b)}})\mid\mathcal{F}_t\bigr]-\nabla F_S(x_t)
=0.
\end{align*}
Thus, the sequence $\{X_b\}_{b=1}^B$ is a martingale difference sequence with respect to $\{\mathcal{G}_{t,b}\}_{b=0}^B$.

By the moment assumption and Jensen's inequality, it holds 
\begin{align*}
\mathbb{E}\bigl[\|X_b\|^p\mid\mathcal{F}_t\bigr]
&\leq 2^{p-1}\left(\mathbb{E}\bigl[\|\nabla f(x_t;\xi_{i_t^{(b)}})\|^p\mid\mathcal{F}_t\bigr]+\|\nabla F_S(x_t)\|^p\right)
\leq 2^pG^p.
\end{align*}
Applying Jensen's inequality and Lemma~\ref{lem:mds_pmoment} conditionally on $\mathcal{F}_t$, we obtain
\begin{align*}
\mathbb{E}\bigl[\|g_t-\nabla F_S(x_t)\|\mid\mathcal{F}_t\bigr]
&\leq \frac{1}{B}\mathbb{E}\left[\left\|\sum_{b=1}^BX_b\right\|^p\middle|\mathcal{F}_t\right]^{\frac{1}{p}} \leq \frac{1}{B}\left(2\sum_{b=1}^B\mathbb{E}\bigl[\|X_b\|^p\mid\mathcal{F}_t\bigr]\right)^{\frac{1}{p}}
\leq 4GB^{-\frac{p-1}{p}}.
\end{align*}
Using the uniform sampling of the output from $\{x_0,\ldots,x_{T-1}\}$ and substituting the above result into \eqref{eq:nsgd_b_optimization}, we have
\begin{align*}
\mathbb{E}_{S,A}\bigl[\|\nabla F_S(A(S))\|\bigr]
&\leq \frac{\Delta}{\eta T}+8GB^{-\frac{p-1}{p}}+\frac{L\eta}{2}.
\end{align*}
Combining this inequality with equation \eqref{eq:nsgd_b_population_reduction} completes the proof.
\end{proof}

\begin{corollary}
\label{cor:nsgd_b_rate}
Following the assumptions of Theorem~\ref{thm:nsgd_b_population}, we take
$T \asymp n^{\frac{2(p-1)}{7p-6}}$,
$\eta\asymp n^{-\frac{p-1}{7p-6}}$,
and
$B\asymp n^{\frac{p}{7p-6}}$.
Then it holds
\begin{align*}
\mathbb{E}_{S,A}\bigl[\|\nabla F(A(S))\|\bigr]
&=O\left(n^{-\frac{p-1}{7p-6}}\right).
\end{align*}
\end{corollary}

\begin{remark}
The orders of parameter settings $\eta\asymp T^{-1/2}$ and $B\asymp T^{p/(2p-2)}$ in Corollary~\ref{cor:nsgd_b_rate} coincide with those derived in the optimization analysis of NSGD-B by \citet{hubler2025gradient}.
\end{remark}

Compared with clipped SGD in Corollary~\ref{cor:clipped_rate}, NSGD-B achieves a sharper population gradient bound.
Note that the stability bound of clipped SGD (Theorem \ref{thm:clipped_population}) scales linearly with $\gamma$, whereas that of NSGD-B (Theorem \ref{thm:nsgd_b_population}) scales with $\sqrt{B}$ which leads to a better result when balancing stability and optimization error.

\begin{remark}
We also study the stability and population gradient bounds for two variants of normalized SGD (NSGD), i.e., NSGD with momentum \citep{cutkosky2020momentum,hubler2025gradient} and NSGD with clipping and momentum \citep{liu2023breaking,sun2025revisiting,cutkosky2021high}.
Please see details in Appendix~\ref{app:smooth_momentum}.
\end{remark}

\section{Generalization Bounds for Nonsmooth Case}
\label{sec:nonsmooth}

In this section, we study the generalization of nonsmooth nonconvex optimization under heavy-tailed noise.
By introducing the on-average gradient stability of the smooth surrogate function, we establish a generalization framework for population Goldstein stationarity and apply it to analyze the online-to-nonconvex conversion method.

\subsection{The Nonsmooth Problem Setting}
\label{subsec:nonsmooth_setting}

We introduce the following Lipschitz assumption for our analysis in the nonsmooth setting.
\begin{assumption}
\label{ass_lipschitz}
For every $\xi\in\Xi$, the stochastic component $f(\cdot;\xi)$ is $L(\xi)$-Lipschitz continuous.
Moreover, there exists some $p\in(1,2]$ such that the random variable $L(\xi)$ has a bounded $p$th moment, i.e., there exists some constant $L>0$ such that
\begin{align*}
\mathbb{E}_{\xi\sim\mathcal{D}}\bigl[L(\xi)^p\bigr] &\leq L^p.
\end{align*}
\end{assumption}

\begin{remark}
In previous work, \citet{lei2023stability} assumes that the stochastic component $f(\cdot;\xi)$ is $L$-Lipschitz continuous over all samples, i.e., there exists some constant $L>0$ such that $|f(x;\xi)-f(y;\xi)|\leq L\|x-y\|$ holds for all $x,y\in\mathbb{R}^d$ and $\xi\in\Xi$.
It is worth noting that our Assumption~\ref{ass_lipschitz} is weaker since it only requires the finite $p$th moment condition which allows $L(\xi)$ to be unbounded and the stochastic gradient to be heavy-tailed when $p<2$.
\end{remark}

It is well-known that finding an approximate Clarke stationary point of a Lipschitz function with finitely many oracle queries is intractable in general \citep{clarke1990optimization,zhang2020complexity}.
Hence, we focus on the stationarity with respect to Goldstein subdifferential which is widely used in the complexity analysis of nonsmooth nonconvex optimization
\citep{goldstein1977optimization,davis2022gradient,lin2022gradient,tian2022finite,cutkosky2023optimal}.

\begin{definition}[Goldstein subdifferential]
\label{def:goldstein_subdifferential}
For given $\delta>0$, we define the Goldstein $\delta$-subdifferential of the Lipschitz continuous function $h:{\mathbb R}^d \to{\mathbb R}$ at point $x\in{\mathbb R}^d$ as
\begin{align*}
\partial_\delta h(x) = \operatorname{conv}\left(\bigcup_{y\in \mathbb B(x,\delta)}\partial h(y)\right),
\end{align*}
where $\mathbb B(x,\delta)=\{y\in\mathbb{R}^d:\|y-x\|\leq\delta\}$ is the closed Euclidean ball centered at point $x\in{\mathbb R}^d$ with radius $\delta>0$ and $\partial h(x)$ is the Clarke subdifferential of $h(\cdot)$ at point  $x\in{\mathbb R}^d$.
\end{definition}
We then introduce the smooth surrogate for the nonsmooth Lipschitz continuous function \citep{flaxman2005online}.
\begin{definition}[Uniform smoothing]
\label{def:uniform_smoothing}
For given $\delta>0$ and a Lipschitz continuous function $h:\mathbb{R}^d\to\mathbb{R}$, we define its uniform smoothing as
\begin{align*}
h_\delta(x) &= \mathbb{E}_{u\sim\operatorname{Unif}(\mathbb B(0,1))}\bigl[h(x+\delta u)\bigr],
\end{align*}
where $\operatorname{Unif}(\mathbb B(0,1))$ is the uniform distribution on the unit Euclidean ball $\mathbb B(0,1)$.
\end{definition}
The following lemma connects the uniform smoothing with Goldstein subdifferential.
\begin{lemma}[\protect{\citet[Proposition~2.3 and Theorem~3.1]{lin2022gradient}}]
\label{lem:smoothing_goldstein}
For given $\delta>0$ and a Lipschitz continuous function $h:\mathbb{R}^d\to\mathbb{R}$, the function $h_\delta(\cdot)$ is differentiable and it holds that~$\nabla h_\delta(x) \in \partial_\delta h(x)$ for all $x\in\mathbb{R}^d$.
\end{lemma}

We use population Goldstein stationarity to measure the generalization of stochastic nonsmooth nonconvex optimization by considering  
\begin{align}\label{eq:E-goldstein}
    \mathbb{E}_{S,A}[\operatorname{dist}(0,\partial_\delta F(A(S)))],
\end{align}
where $\operatorname{dist}(0,\partial_\delta F(A(S)))=\inf_{g\in\partial_\delta F(A(S))}||g||$, and the
notations $F(\cdot)$ and $A(S)$ follow problem~(\ref{eq_prm}) and Section \ref{subsec:smooth_setting}.
Applying Lemma~\ref{lem:smoothing_goldstein} to functions $F(\cdot)$ and $F_S(\cdot)$ gives
\begin{align*}
\nabla F_\delta(x) &\in\partial_\delta F(x)
\qquad \text{and} \qquad
\nabla (F_S)_\delta(x)\in\partial_\delta F_S(x),
\end{align*}
which implies 
$\operatorname{dist}(0,\partial_\delta F(A(S)))\leq\|\nabla F_\delta(A(S))\|$.
Therefore, the triangle inequality yields
\begin{align}
\begin{split}
\mathbb{E}_{S,A}\bigl[\operatorname{dist}(0,\partial_\delta F(A(S)))\bigr]
&\leq \underbrace{\mathbb{E}_{S,A}\bigl[\|\nabla (F_S)_\delta(A(S))\|\bigr]}_{\text{optimization error}} \\
& \quad +\underbrace{\mathbb{E}_{S,A}\bigl[\|\nabla F_\delta(A(S))-\nabla (F_S)_\delta(A(S))\|\bigr]}_{\text{generalization error}}.
\label{eq:nonsmooth_error_decomposition}
\end{split}
\end{align}
In later subsections, we introduce on-average gradient stability to bound the generalization error, then analyze the generalization of online-to-nonconvex conversion methods with respect to Goldstein stationarity.

\subsection{Generalization Error for the Nonsmooth Case}
\label{subsec:nonsmooth_generalization}

For nonsmooth nonconvex problems, we introduce the following on-average gradient stability for the smooth surrogate function.    

\begin{definition}[On-average gradient stability]
\label{def:on_average_gradient_stability}

A randomized algorithm $A$ is $\epsilon_\delta$-on-average gradient stable if
\begin{align*}
\frac{1}{n}\sum_{i=1}^n\mathbb{E}_{S,S',A,\xi}\bigl[\|\nabla f_\delta(A(S);\xi)-\nabla f_\delta(A(S^{(i)});\xi)\|^p\bigr] &\leq \epsilon_\delta^p, \\
\frac{1}{n}\sum_{i=1}^n\mathbb{E}_{S,S',A}\bigl[\|\nabla f_\delta(A(S);\xi_i)-\nabla f_\delta(A(S^{(i)});\xi_i)\|\bigr] &\leq \epsilon_\delta,
\end{align*}
where $S=\{\xi_1,\ldots,\xi_n\}$ and $S'=\{\xi_1',\ldots,\xi_n'\}$ are two datasets such that all of their samples are drawn independently from $\mathcal{D}$, 
and $S^{(i)}=\{\xi_1,\ldots,\xi_{i-1},\xi_i',\xi_{i+1},\ldots,\xi_n\}$. 

\end{definition}

\begin{remark}
Note that our Assumption~\ref{ass_lipschitz} only requires  $\mathbb{E}_{\xi\sim\mathcal{D}}\bigl[L(\xi)^p\bigr] \leq L^p$, while it allows the individual $L(\xi)$ to be unbounded for some $\xi\in\Xi$ so that the gradient of $f_\delta(\cdot;\xi)$ also may be unbounded.
Therefore, we cannot directly follow Definition \ref{def:gradient_stability} to define the uniform gradient stability by the upper bound of $\sup_{\xi\in\Xi}\mathbb{E}_{A}\bigl[\|\nabla f_\delta(A(S);\xi)-\nabla f_\delta(A(S^{(i)});\xi)\|^2\bigr]$.
\end{remark}

We now bound the generalization error term in equation \eqref{eq:nonsmooth_error_decomposition} based on the on-average gradient stability.

\begin{theorem}
\label{thm:nonsmooth_generalization}
Suppose that the randomized algorithm $A$ is $\epsilon_\delta$-on-average gradient stable and Assumption~\ref{ass_lipschitz} holds. Then we have
\begin{align*}
\mathbb{E}_{S,A}\bigl[\|\nabla F_\delta(A(S))-\nabla (F_S)_\delta(A(S))\|\bigr]
&\leq 2\epsilon_\delta+C_p(2L)^{2-p}\left(\frac{(2L)^p}{n}+4\epsilon_\delta^p\right)^{\frac{p-1}{p}},
\end{align*}
where $C_2=1$ and $C_p=\frac{p}{2(p-1)}\left(\frac{4(p-1)}{2-p}\right)^{\frac{2-p}{p}}$ for $p\in(1,2)$.
\end{theorem}

\begin{proof}
Similar to the analysis in the smooth setting, 
we use the truncation argument to establish the generalization error bound according to the following four steps.

\noindent\textbf{Step 1: Decomposition and truncation.}
Following the notation in Section \ref{sec:smooth}, we let~$S'=\{\xi_1',\ldots,\xi_n'\}$ such that all of its elements are drawn independently from distribution~$\mathcal{D}$ and define
\begin{align*}
S^{(i)}=\{\xi_1,\ldots,\xi_{i-1},\xi_i',\xi_{i+1},\ldots,\xi_n\}
\end{align*} 
for each $i\in[n]$.
We decompose the gradient discrepancy as 
\begin{align}
\begin{split}
& n\bigl(\nabla F_\delta(A(S))-\nabla (F_S)_\delta(A(S))\bigr) \\
=&\sum_{i=1}^n\mathbb{E}_{\xi,\xi_i'}\bigl[\nabla f_\delta(A(S);\xi)-\nabla f_\delta(A(S^{(i)});\xi)\bigr] \\
&+\sum_{i=1}^n\mathbb{E}_{\xi_i'}\bigl[\nabla F_\delta(A(S^{(i)}))-\nabla f_\delta(A(S^{(i)});\xi_i)\bigr] \\
&+\sum_{i=1}^n\mathbb{E}_{\xi_i'}\bigl[\nabla f_\delta(A(S^{(i)});\xi_i)-\nabla f_\delta(A(S);\xi_i)\bigr].
\label{eq:nonsmooth_ghost_decomposition}
\end{split}
\end{align}
We define
\begin{align*}
z_i(S) &= \mathbb{E}_{\xi_i'}\bigl[\nabla F_\delta(A(S^{(i)}))-\nabla f_\delta(A(S^{(i)});\xi_i)\bigr].
\end{align*}
Taking norms and expectations in equation \eqref{eq:nonsmooth_ghost_decomposition}, we obtain
\begin{align}\label{eq:nonsmooth_triangle_bound}
\begin{split}
& n\mathbb{E}_{S,A}\bigl[\|\nabla F_\delta(A(S))-\nabla (F_S)_\delta(A(S))\|\bigr] \\
\leq & \sum_{i=1}^n\mathbb{E}_{S,S',A,\xi}\bigl[\|\nabla f_\delta(A(S);\xi)-\nabla f_\delta(A(S^{(i)});\xi)\|\bigr] \\
& +\mathbb{E}_{S,A}\left[\left\|\sum_{i=1}^n z_i(S)\right\|\right] \\
& +\sum_{i=1}^n\mathbb{E}_{S,S',A}\bigl[\|\nabla f_\delta(A(S^{(i)});\xi_i)-\nabla f_\delta(A(S);\xi_i)\|\bigr].
\end{split}
\end{align}
By Jensen's inequality and the first condition in Definition~\ref{def:on_average_gradient_stability}, it holds
\begin{align}\label{eq:nonsmooth_fresh_stability}
\begin{split}    
&\quad \frac{1}{n}\sum_{i=1}^n\mathbb{E}_{S,S',A,\xi}\bigl[\|\nabla f_\delta(A(S);\xi)-\nabla f_\delta(A(S^{(i)});\xi)\|\bigr] \\
&\leq \left(\frac{1}{n}\sum_{i=1}^n\mathbb{E}_{S,S',A,\xi}\bigl[\|\nabla f_\delta(A(S);\xi)-\nabla f_\delta(A(S^{(i)});\xi)\|^p\bigr]\right)^{\frac{1}{p}} \\
&\leq \epsilon_\delta.
\end{split}
\end{align}
Consequently, the second condition in Definition~\ref{def:on_average_gradient_stability} ensures that
\begin{align}\label{eq:nonsmooth_training_stability}
\frac{1}{n}\sum_{i=1}^n\mathbb{E}_{S,S',A}\bigl[\|\nabla f_\delta(A(S^{(i)});\xi_i)-\nabla f_\delta(A(S);\xi_i)\|\bigr]
&\leq \epsilon_\delta.
\end{align}
Dividing equation \eqref{eq:nonsmooth_triangle_bound} by $n$ and using equations \eqref{eq:nonsmooth_fresh_stability} and \eqref{eq:nonsmooth_training_stability}, one has
\begin{align}\label{eq:nonsmooth_z_reduction}
\mathbb{E}_{S,A}\bigl[\|\nabla F_\delta(A(S))-\nabla (F_S)_\delta(A(S))\|\bigr]
&\leq 2\epsilon_\delta+\frac{1}{n}\mathbb{E}_{S,A}\left[\left\|\sum_{i=1}^n z_i(S)\right\|\right].
\end{align}

We construct a centered truncation by defining
\begin{align*}
T_\tau(x;\xi) &= \mathrm{clip}_\tau\bigl(\nabla f_\delta(x;\xi)-\nabla F_\delta(x)\bigr),\\
M_\tau(x) &= \mathbb{E}_\xi\bigl[T_\tau(x;\xi)\bigr],\\
\widetilde{T}_\tau(x;\xi) &= T_\tau(x;\xi)-M_\tau(x).
\end{align*}
For all $i\in[n]$, we define
\begin{align*}
z_i^{(\tau)}(S) &= -\mathbb{E}_{\xi_i'}\bigl[\widetilde{T}_\tau(A(S^{(i)});\xi_i)\bigr]
\qquad \text{and} \qquad
r_i^{(\tau)}(S)=z_i(S)-z_i^{(\tau)}(S).
\end{align*}
By the definition of $r_i^{(\tau)}(S)$ and the triangle inequality, we have
\begin{align*}
\left\|\sum_{i=1}^nz_i(S)\right\|
&\leq \left\|\sum_{i=1}^nz_i^{(\tau)}(S)\right\|+\left\|\sum_{i=1}^nr_i^{(\tau)}(S)\right\|.
\end{align*}

\noindent\textbf{Step 2: Bound the term $\mathbb{E}_{S,A}\big[\big\|\sum_{i=1}^n z_i^{(\tau)}(S)\big\|\big]$.}

Based on the definition of $z_i^{(\tau)}(S)$, it holds
$\mathbb{E}_{\xi_i}\bigl[z_i^{(\tau)}(S)\bigr]=0$ and
\begin{align}\label{eq:nonsmooth_ztS}
\mathbb{E}_{S,A}\left[\left\|\sum_{i=1}^n z_i^{(\tau)}(S)\right\|^2\right]
&=\sum_{i=1}^n\mathbb{E}_{S,A}\bigl[\|z_i^{(\tau)}(S)\|^2\bigr]
+\sum_{i\neq j}\mathbb{E}_{S,A}\bigl[\langle z_i^{(\tau)}(S),z_j^{(\tau)}(S)\rangle\bigr].
\end{align}
We first consider the first term on the right-hand side of equation \eqref{eq:nonsmooth_ztS}. It holds
\begin{align}\label{eq:nonsmooth_diagonal_bound}
\mathbb{E}_{S,A}\bigl[\|z_i^{(\tau)}(S)\|^2\bigr]
&\leq \mathbb{E}_{S,S',A,\xi}\bigl[\|\widetilde{T}_\tau(A(S^{(i)});\xi)\|^2\bigr] \notag\\
&\leq \mathbb{E}_{S,S',A,\xi}\bigl[\|T_\tau(A(S^{(i)});\xi)\|^2\bigr]\notag \\
&\leq \tau^{2-p}(2L)^p,
\end{align}
where the first inequality follows from Jensen's inequality and the second from the inequality $\mathbb{E}\|X-\mathbb{E}X\|^2\leq\mathbb{E}\|X\|^2$. Here the last step is based on equation \eqref{eq:clipping_properties} and the fact
\begin{align}
\mathbb{E}_{\xi\sim\mathcal{D}}\bigl[\|\nabla f_\delta(x;\xi)-\nabla F_\delta(x)\|^p\bigr] &\leq 2^p\mathbb{E}_{\xi\sim\mathcal{D}}\bigl[\|\nabla f_\delta(x;\xi)\|^p\bigr]\leq (2L)^p,
\label{eq:nonsmooth_smoothed_moment}
\end{align}
where the first inequality follows from Jensen's inequality, and the second from Assumption~\ref{ass_lipschitz}.

We then consider the second term on the right-hand side of equation \eqref{eq:nonsmooth_ztS}.
We use the notation $S''=\{\xi_1'',\ldots,\xi_n''\}$ to denote another independent copy of $S$.
Accordingly, we define
\begin{align*}
S_j &= \{\xi_1,\dots,\xi_{j-1},\xi_j'',\xi_{j+1},\dots,\xi_n\},\\
S_j^{(i)} &= \{\xi_1,\dots,\xi_{i-1},\xi_i',\xi_{i+1},\dots,\xi_{j-1},\xi_j'',\xi_{j+1},\dots,\xi_n\}.
\end{align*}
In the case of $i\neq j$, the centering of $\widetilde{T}_\tau$ ensures
\begin{align}
\mathbb{E}_{S,S'',A}\bigl[\langle z_i^{(\tau)}(S_j),z_j^{(\tau)}(S)\rangle\bigr] &= 0, \notag\\
\mathbb{E}_{S,S'',A}\bigl[\langle z_i^{(\tau)}(S),z_j^{(\tau)}(S_i)\rangle\bigr] &= 0, \notag\\
\mathbb{E}_{S,S'',A}\bigl[\langle z_i^{(\tau)}(S_j),z_j^{(\tau)}(S_i)\rangle\bigr] &= 0. \notag
\end{align}
Using the above equations and Young's inequality, we obtain
\begin{align}\label{eq:nonsmooth_cross_young}
\begin{split}
&\quad \mathbb{E}_{S,A}\bigl[\langle z_i^{(\tau)}(S),z_j^{(\tau)}(S)\rangle\bigr]\\
&= \mathbb{E}_{S,S'',A}\Bigl[\bigl\langle z_i^{(\tau)}(S)-z_i^{(\tau)}(S_j),z_j^{(\tau)}(S)-z_j^{(\tau)}(S_i)\bigr\rangle\Bigr] \\
&\leq \frac{1}{2}\mathbb{E}_{S,S'',A}\bigl[\|z_i^{(\tau)}(S)-z_i^{(\tau)}(S_j)\|^2\bigr]+\frac{1}{2}\mathbb{E}_{S,S'',A}\bigl[\|z_j^{(\tau)}(S)-z_j^{(\tau)}(S_i)\|^2\bigr].
\end{split}
\end{align}
By Jensen's inequality, we have
\begin{align}\label{eq:nonsmooth_difference_jensen}
\|z_i^{(\tau)}(S)-z_i^{(\tau)}(S_j)\|^2 &\leq \mathbb{E}_{\xi_i'}\bigl[\|\widetilde{T}_\tau(A(S^{(i)});\xi_i)-\widetilde{T}_\tau(A(S_j^{(i)});\xi_i)\|^2\bigr].
\end{align}
Moreover, for fixed $x,y\in\mathbb{R}^d$, it holds
\begin{align}\label{eq:nonsmooth_centering_bound}
\mathbb{E}_\xi\bigl[\|\widetilde{T}_\tau(x;\xi)-\widetilde{T}_\tau(y;\xi)\|^2\bigr] &\leq \mathbb{E}_\xi\bigl[\|T_\tau(x;\xi)-T_\tau(y;\xi)\|^2\bigr],
\end{align}
where the inequality follows from the definition of $M_\tau$ and the fact $\mathbb{E}\|X-\mathbb{E}X\|^2\leq\mathbb{E}\|X\|^2$.
By Jensen's inequality, it holds that
\begin{align}\label{eq:nonsmooth_centered_p_bound}
\begin{split}    
& \mathbb{E}_\xi\bigl[\|(\nabla f_\delta(x;\xi)-\nabla F_\delta(x))-(\nabla f_\delta(y;\xi)-\nabla F_\delta(y))\|^p\bigr] \\
\leq & 2^p\mathbb{E}_\xi\bigl[\|\nabla f_\delta(x;\xi)-\nabla f_\delta(y;\xi)\|^p\bigr].
\end{split}
\end{align}
Combining equation \eqref{eq:nonsmooth_centering_bound}, the bound $\|T_\tau(x;\xi)-T_\tau(y;\xi)\|\leq 2\tau$, and the $1$-Lipschitz continuity of the clipping operator, we have
\begin{align}\label{eq:nonsmooth_clipping_difference}
\begin{split}    
& \mathbb{E}_\xi\bigl[\|\widetilde{T}_\tau(x;\xi)-\widetilde{T}_\tau(y;\xi)\|^2\bigr] \\ \leq & (2\tau)^{2-p}\mathbb{E}_\xi\bigl[\|(\nabla f_\delta(x;\xi)-\nabla F_\delta(x))-(\nabla f_\delta(y;\xi)-\nabla F_\delta(y))\|^p\bigr] \\
\leq & 4\tau^{2-p}\mathbb{E}_\xi\bigl[\|\nabla f_\delta(x;\xi)-\nabla f_\delta(y;\xi)\|^p\bigr],
\end{split}
\end{align}
where the last inequality follows from equation \eqref{eq:nonsmooth_centered_p_bound} and $(2\tau)^{2-p}2^p=4\tau^{2-p}$.

Taking expectation in equation \eqref{eq:nonsmooth_difference_jensen} and applying equation \eqref{eq:nonsmooth_clipping_difference} with $x=A(S^{(i)})$ and~$y=A(S_j^{(i)})$, one has
\begin{align}\label{eq:nonsmooth_difference_bound}
\begin{split}    
& \mathbb{E}_{S,S'',A}\bigl[\|z_i^{(\tau)}(S)-z_i^{(\tau)}(S_j)\|^2\bigr] \\
\leq & 4\tau^{2-p}\mathbb{E}_{S,S',S'',A}\bigl[\|\nabla f_\delta(A(S^{(i)});\xi_i)-\nabla f_\delta(A(S_j^{(i)});\xi_i)\|^p\bigr].
\end{split}
\end{align}
For each fixed $i$, the dataset $S^{(i)}$ has the same distribution as that of $S$.
Moreover, dataset~$S_j^{(i)}$ is obtained from $S^{(i)}$ by replacing one sample, while $\xi_i$ is independent of both datasets.
The first condition in Definition~\ref{def:on_average_gradient_stability} therefore yields
\begin{align}\label{eq:nonsmooth_stability_sum}
\sum_{j\neq i}\mathbb{E}_{S,S',S'',A}\bigl[\|\nabla f_\delta(A(S^{(i)});\xi_i)-\nabla f_\delta(A(S_j^{(i)});\xi_i)\|^p\bigr] &\leq n\epsilon_\delta^p.
\end{align}
Summing equation \eqref{eq:nonsmooth_cross_young} over $i\neq j$ and interchanging $i$ and $j$ in the second sum, we obtain
\begin{align}\label{eq:nonsmooth_cross_bound}
\begin{split}    
& \sum_{i\neq j}\mathbb{E}_{S,A}\bigl[\langle z_i^{(\tau)}(S),z_j^{(\tau)}(S)\rangle\bigr] \\
\leq & \sum_{i\neq j}\mathbb{E}_{S,S'',A}\bigl[\|z_i^{(\tau)}(S)-z_i^{(\tau)}(S_j)\|^2\bigr] \\
\leq & 4\tau^{2-p}\sum_{i=1}^n\sum_{j\neq i}\mathbb{E}_{S,S',S'',A}\bigl[\|\nabla f_\delta(A(S^{(i)});\xi_i)-\nabla f_\delta(A(S_j^{(i)});\xi_i)\|^p\bigr] \\
\leq & 4n^2\tau^{2-p}\epsilon_\delta^p,
\end{split}
\end{align}
where the last two inequalities follow from equations \eqref{eq:nonsmooth_difference_bound} and \eqref{eq:nonsmooth_stability_sum}, respectively.

Applying Jensen's inequality and using equations \eqref{eq:nonsmooth_ztS}, \eqref{eq:nonsmooth_diagonal_bound}, and \eqref{eq:nonsmooth_cross_bound}, we~have
\begin{align}\label{eq:nonsmooth_truncated_sum}
\frac{1}{n}\mathbb{E}_{S,A}\left[\left\|\sum_{i=1}^n z_i^{(\tau)}(S)\right\|\right] &\leq \sqrt{\frac{\tau^{2-p}(2L)^p}{n}+4\tau^{2-p}\epsilon_\delta^p}.
\end{align}

\noindent\textbf{Step 3: Bound the term $\mathbb{E}_{S,A}\big[\big\|\sum_{i=1}^n r_i^{(\tau)}(S)\big\|\big]$.}

From the definitions of $z_i(S)$ and $z_i^{(\tau)}(S)$ and the identity $\mathbb{E}_\xi[\nabla f_\delta(x;\xi)-\nabla F_\delta(x)]=0$, one has
\begin{align}
r_i^{(\tau)}(S) &=-\mathbb{E}_{\xi_i'}\bigl[\nabla f_\delta(A(S^{(i)});\xi_i)-\nabla F_\delta(A(S^{(i)}))-T_\tau(A(S^{(i)});\xi_i)\bigr] \notag\\
&\quad+\mathbb{E}_{\xi_i'}\mathbb{E}_\xi\bigl[\nabla f_\delta(A(S^{(i)});\xi)-\nabla F_\delta(A(S^{(i)}))-T_\tau(A(S^{(i)});\xi)\bigr]. \notag
\end{align}
Applying the triangle inequality and Jensen's inequality to the above decomposition, we~have
\begin{align}
\mathbb{E}_{S,A}\bigl[\|r_i^{(\tau)}(S)\|\bigr] &\leq 2\mathbb{E}_{S,S',A,\xi}\bigl[\|\nabla f_\delta(A(S^{(i)});\xi)-\nabla F_\delta(A(S^{(i)}))-T_\tau(A(S^{(i)});\xi)\|\bigr]\leq \frac{2(2L)^p}{\tau^{p-1}}, \notag
\end{align}
where the first inequality also uses the independence of $\xi_i$ and $A(S^{(i)})$, and the second follows from the residual bound in equation \eqref{eq:clipping_properties} and equation \eqref{eq:nonsmooth_smoothed_moment}.
Applying the triangle inequality, we obtain
\begin{align}\label{eq:nonsmooth_residual_bound}
\frac{1}{n}\mathbb{E}_{S,A}\left[\left\|\sum_{i=1}^n r_i^{(\tau)}(S)\right\|\right] &\leq \frac{2(2L)^p}{\tau^{p-1}}.
\end{align}

\noindent\textbf{Step 4: Choose the value of $\tau$.}

By the triangle inequality, the identity $z_i(S)=z_i^{(\tau)}(S)+r_i^{(\tau)}(S)$, and equations \eqref{eq:nonsmooth_z_reduction}, \eqref{eq:nonsmooth_truncated_sum}, and \eqref{eq:nonsmooth_residual_bound}, we have
\begin{align}\label{eq:nonsmooth_phi-12}
\mathbb{E}_{S,A}\bigl[\|\nabla F_\delta(A(S))-\nabla (F_S)_\delta(A(S))\|\bigr] &\leq 2\epsilon_\delta+\left(\frac{(2L)^p}{n}+4\epsilon_\delta^p\right)^{\frac{1}{2}}\tau^{\frac{2-p}{2}}+2(2L)^p\tau^{1-p}.
\end{align}
In the case of $p=2$, taking $\tau\to+\infty$ in equation \eqref{eq:nonsmooth_phi-12} yields
\begin{align}\label{eq:nonsmooth_phi-1}
\mathbb{E}_{S,A}\bigl[\|\nabla F_\delta(A(S))-\nabla (F_S)_\delta(A(S))\|\bigr] &\leq 2\epsilon_\delta+\left(\frac{(2L)^2}{n}+4\epsilon_\delta^2\right)^{\frac{1}{2}}.
\end{align}
In the case of $p\in(1,2)$, we define
\begin{align}
\phi(\tau) &= \left(\frac{(2L)^p}{n}+4\epsilon_\delta^p\right)^{\frac{1}{2}}\tau^{\frac{2-p}{2}}+2(2L)^p\tau^{1-p}, \notag
\end{align}
where $\tau>0$.
Then we have
\begin{align}
\phi'(\tau) &= \frac{2-p}{2}\left(\frac{(2L)^p}{n}+4\epsilon_\delta^p\right)^{\frac{1}{2}}\tau^{-\frac{p}{2}}+2(2L)^p(1-p)\tau^{-p}. \notag
\end{align}
Note that the equation $\phi'(\tau)=0$ has the unique solution
\begin{align}
\tau_* &= \left(\frac{4(p-1)(2L)^p}{(2-p)\left(\frac{(2L)^p}{n}+4\epsilon_\delta^p\right)^{\frac{1}{2}}}\right)^{\frac{2}{p}}, \notag
\end{align}
and it holds that $\phi'(\tau)<0$ for $\tau\in(0,\tau_*)$ and $\phi'(\tau)>0$ for $\tau>\tau_*$.
Hence, the point $\tau_*$ is the unique global minimizer of $\phi(\cdot)$ over $(0,\infty)$.
It follows that
\begin{align}\label{eq:nonsmooth_phi-star}
\phi(\tau_*) &= \frac{p}{2(p-1)}\left(\frac{4(p-1)}{2-p}\right)^{\frac{2-p}{p}}(2L)^{2-p}\left(\frac{(2L)^p}{n}+4\epsilon_\delta^p\right)^{\frac{p-1}{p}} \notag\\
&= C_p(2L)^{2-p}\left(\frac{(2L)^p}{n}+4\epsilon_\delta^p\right)^{\frac{p-1}{p}}.
\end{align}
Combining equations \eqref{eq:nonsmooth_phi-12}, \eqref{eq:nonsmooth_phi-1}, and \eqref{eq:nonsmooth_phi-star}, for every $p\in(1,2]$, it holds
\begin{align}
\mathbb{E}_{S,A}\bigl[\|\nabla F_\delta(A(S))-\nabla (F_S)_\delta(A(S))\|\bigr] &\leq 2\epsilon_\delta+C_p(2L)^{2-p}\left(\frac{(2L)^p}{n}+4\epsilon_\delta^p\right)^{\frac{p-1}{p}}, \notag
\end{align}
which completes the proof.
\end{proof}

\begin{remark}
It is worth noting that the assumptions required in our generalization error analysis are more general than those in existing works on nonsmooth nonconvex optimization.  
\begin{enumerate}[leftmargin=0.6cm,topsep=-0.1cm,itemsep=-0.1cm]
\item The results of \citet{davis2022graphical} and \citet{lei2023stability} require that the Moreau envelopes be well-defined, which depends on the weak convexity of the objective, which is not necessary in our analysis.
\item The results of \citet{lei2023stability} and \citet{kornowski2025improved} depend on the uniform Lipschitz continuity assumption, while our analysis only depends on the weaker condition of a bounded $p$th moment.
Additionally, our condition is more general than that of \citet{davis2022graphical} which is limited to the special case of Assumption \ref{ass_lipschitz} with~$p=2$. 
\end{enumerate}
\end{remark}

\subsection{Application to Online-to-Nonconvex Conversion}
\label{subsec:nonsmooth_o2nc}

The online-to-nonconvex conversion (O2NC) is a well-known technique for stochastic nonsmooth nonconvex optimization \citep{cutkosky2023optimal}.
From an optimization perspective,  O2NC achieves the optimal stochastic first-order oracle (SFO) complexity for finding approximate Goldstein stationary points under bounded-variance noise \citep{cutkosky2023optimal,kornowski2026complexity}.
It also achieves the best-known SFO complexity under heavy-tailed noise \citep{liu2026online}.
In this subsection, we study the generalization of solving empirical risk problem~(\ref{eq_erm}) by the O2NC method, which is presented in Algorithm~\ref{alg:smoothed_o2nc}.
We first establish its stability as follows.

\begin{proposition}
\label{prop:o2nc_stability}
Suppose that Assumption~\ref{ass_lipschitz} holds. 
Then Algorithm~\ref{alg:smoothed_o2nc} is $\epsilon_{\mathrm{O2NC}}$-on-average gradient stable with
\begin{align}
\epsilon_{\mathrm{O2NC}} &\leq 2L\left(\frac{KT}{n}\right)^{\frac{1}{p}}. \notag
\end{align}
Let $z_m(S)$ and $z_m(S^{(i)})$ be the points $z_m$ achieved by running Algorithm~\ref{alg:smoothed_o2nc} on datasets $S$ and~$S^{(i)}$, respectively. 
Then for all $m\in[KT]$, it holds
\begin{align}
\label{eq:nonsmooth_stability_zm1}
\frac{1}{n}\sum_{i=1}^n\mathbb{E}_{S,S',A,\xi}\bigl[\|\nabla f_r(z_m(S);\xi)-\nabla f_r(z_m(S^{(i)});\xi)\|^p\bigr] &\leq \epsilon_{\mathrm{O2NC}}^p, 
\end{align}
\vspace{-0.6cm}
\begin{align}
\frac{1}{n}\sum_{i=1}^n\mathbb{E}_{S,S',A}\bigl[\|\nabla f_r(z_m(S);\xi_i)-\nabla f_r(z_m(S^{(i)});\xi_i)\|\bigr] &\leq \epsilon_{\mathrm{O2NC}}. 
\label{eq:nonsmooth_stability_zm2}
\end{align}
\end{proposition}

\begin{proof}
Let $I(A)=\{i_m\}_{m=1}^{KT}$ denote the sequence consisting of the indices of samples.
For each $i\in[n]$, we run Algorithm~\ref{alg:smoothed_o2nc} on datasets $S$ and $S^{(i)}$ with the same $I(A)$, $\{s_m\}_{m=1}^{KT}$, $\{u_m\}_{m=1}^{KT}$, and $R$.
We use $E_i=\{i\in I(A)\}$ to represent the event that index $i$ is sampled at least once. It follows that
\begin{align}
\mathbb{P}(E_i) &\leq \min\left\{1,\frac{KT}{n}\right\}. \label{eq:o2nc_sampling_probability}
\end{align}
On the event $\{i\notin I(A)\}$, the two runs generate identical iterates, which implies that $A(S)=A(S^{(i)})$ and $z_m(S)=z_m(S^{(i)})$ for all $m\in[KT]$.
Additionally, Assumption~\ref{ass_lipschitz} implies that $\|\nabla f_r(x;\xi)\|\leq L(\xi)$ for all $x\in\mathbb{R}^d$ and $\xi\in\Xi$.
Hence, for every $\xi\in\Xi$, it holds
\begin{align}
\mathbb{E}_A\bigl[\|\nabla f_r(A(S);\xi)-\nabla f_r(A(S^{(i)});\xi)\|\bigr] &\leq 2L(\xi)\mathbb{P}(E_i), \label{eq:o2nc_output_moment}\\
\mathbb{E}_A\bigl[\|\nabla f_r(A(S);\xi)-\nabla f_r(A(S^{(i)});\xi)\|^p\bigr] &\leq (2L(\xi))^p\mathbb{P}(E_i). \label{eq:o2nc_output_pmoment}
\end{align}
Since the randomness of event $E_i$ only depends on $I(A)$, taking expectation on equation~\eqref{eq:o2nc_output_pmoment} and using Assumption~\ref{ass_lipschitz} and equation \eqref{eq:o2nc_sampling_probability} yield
\begin{align}
\begin{split}
&\quad \frac{1}{n}\sum_{i=1}^n\mathbb{E}_{S,S',A,\xi}\bigl[\|\nabla f_r(A(S);\xi)-\nabla f_r(A(S^{(i)});\xi)\|^p\bigr] \\
&\leq \frac{2^p}{n}\sum_{i=1}^n\mathbb{E}_{\xi}[L(\xi)^p]\mathbb{P}(E_i) \leq (2L)^p\frac{KT}{n}. 
\label{eq:o2nc_stability_p}
\end{split}
\end{align}
Applying equation \eqref{eq:o2nc_output_moment} with  $\xi=\xi_i$ and using $\mathbb{E}_{\xi_i}[L(\xi_i)]\leq L$ achieves
\begin{align}
\begin{split}
\label{eq:o2nc_stability_1}
& \quad \frac{1}{n}\sum_{i=1}^n\mathbb{E}_{S,S',A}\bigl[\|\nabla f_r(A(S);\xi_i)-\nabla f_r(A(S^{(i)});\xi_i)\|\bigr] \\
&\leq \frac{2}{n}\sum_{i=1}^n\mathbb{E}_{\xi_i}[L(\xi_i)]\mathbb{P}(E_i) \leq 2L\min\left\{1,\frac{KT}{n}\right\} \leq 2L\left(\frac{KT}{n}\right)^{\frac{1}{p}}. 
\end{split}
\end{align}
Hence, Algorithm~\ref{alg:smoothed_o2nc} is $\epsilon_{\mathrm{O2NC}}$-on-average gradient stable with~$\epsilon_{\mathrm{O2NC}}=2L(KT/n)^{1/p}$.    

Similarly, following the derivations of equations \eqref{eq:o2nc_output_moment}--\eqref{eq:o2nc_stability_1} by replacing $A(S)$ and $A(S^{(i)})$ with $z_m(S)$ and $z_m(S^{(i)})$, we achieve inequalities \eqref{eq:nonsmooth_stability_zm1} and \eqref{eq:nonsmooth_stability_zm2}. 
\end{proof}

\begin{algorithm}[t]
\caption{Online-to-Nonconvex Conversion with Online Gradient Descent}
\label{alg:smoothed_o2nc}
\begin{algorithmic}[1]
\STATE \textbf{Input:} Initial point $y_0$, stepsize $\eta$, smoothing radius $r$, localization radius $\rho>0$, epoch number $K$, epoch length $T$
\FOR{$k=1,2,\ldots,K$}
    \STATE $x_{(k-1)T+1}=0$
    \FOR{$t=1,2,\ldots,T$}
        \STATE $m=(k-1)T+t$ \\[0.1cm]
        \STATE $s_m\sim\operatorname{Unif}([0,1])$ \\[0.1cm]
        \STATE $z_m=y_{m-1}+s_mx_m$ \\[0.1cm]
        \STATE $i_m\sim\operatorname{Unif}([n])$ \\[0.1cm]
        \STATE $u_m\sim\operatorname{Unif}(\mathbb B(0,1))$ \\[0.1cm]
        \STATE $v_m=\nabla f(z_m+ru_m;\xi_{i_m})$ \\[0.1cm]
        \STATE $y_m=y_{m-1}+x_m$ \\[0.1cm]
        \STATE $x_{m+1}=\Pi_{\mathbb B(0,\rho/T)}\left(x_m-\eta v_m\right)$ \\[0.1cm]
    \ENDFOR
    \STATE $\displaystyle{\bar z_k=\frac{1}{T}\sum_{m=(k-1)T+1}^{kT}z_m}$
\ENDFOR
\STATE \textbf{Output:} $\bar z_R$, where $R$ is uniformly sampled from $\{1,2, \dots,K\}$
\end{algorithmic}
\end{algorithm}

We then establish the following bound on the inner products for each epoch.

\begin{lemma}
\label{lem:ogd_pathwise_regret}
For all $p\in(1,2]$, $k\in[K]$, and $w_k\in\mathbb B(0,\rho/T)$, 
Algorithm \ref{alg:smoothed_o2nc} satisfies
\begin{align}
\sum_{m=(k-1)T+1}^{kT}\langle v_m,x_m-w_k\rangle &\leq \frac{\rho^2}{2\eta T^2}+2\left(\frac{\rho}{T}\right)^{2-p}\eta^{p-1}\sum_{m=(k-1)T+1}^{kT}\|v_m\|^p. \notag
\end{align}
\end{lemma}

\begin{remark}
The proof of Lemma \ref{lem:ogd_pathwise_regret} follows \citet[Theorem~1]{liu2026online} but considers the general stepsize $\eta>0$.
We present the detailed proof in Appendix \ref{app:lemma_proof}.
\end{remark}

Combining Proposition~\ref{prop:o2nc_stability} with the analysis of optimization error gives the following population Goldstein stationarity bound.

\begin{theorem}
\label{thm:o2nc_population}
Consider Algorithm~\ref{alg:smoothed_o2nc} with constant stepsize $\eta=\rho/(LT^{1+1/p})$, smoothing radius~$r>0$, and localization radius~$\rho>0$, and initial point~$y_0$.
Suppose Assumption~\ref{ass_lipschitz} holds and $\Delta=\mathbb{E}_S[F_S(y_0)-\inf_{x\in\mathbb{R}^d}F_S(x)]<+\infty$. Then we have
\begin{align*}
\mathbb{E}_{S,A}\bigl[\operatorname{dist}(0,\partial_{r+\rho}F(A(S)))\bigr] &\leq \frac{\Delta+rL}{\rho K}+\frac{7L}{T^{\frac{p-1}{p}}}+2\epsilon_{\mathrm{O2NC}}+C_p(2L)^{2-p}\biggl(\frac{(2L)^p}{n}+4\epsilon_{\mathrm{O2NC}}^p\biggr)^{\frac{p-1}{p}}\!\!, 
\end{align*}
where $A(S)$ is the output of Algorithm~\ref{alg:smoothed_o2nc}, $\epsilon_{\mathrm{O2NC}}$ is the stability bound given in Proposition~\ref{prop:o2nc_stability}, and $C_p>0$ is the constant defined in Theorem~\ref{thm:nonsmooth_generalization}.
\end{theorem}

\begin{proof}
We first bound the expected norm of the average gradient of the smooth  function 
\begin{align*}
(F_S)_r(x) = \mathbb{E}_{u\sim\operatorname{Unif}(\mathbb B(0,1))}\bigl[F_S(x+r u)\bigr].
\end{align*}
For each $k\in[K]$, we define
\begin{align}
w_k =-\frac{\rho\sum_{m=(k-1)T+1}^{kT}\nabla(F_S)_r(z_m)}{T\bigl\|\sum_{m=(k-1)T+1}^{kT}\nabla(F_S)_r(z_m)\bigr\|}, \notag
\end{align}
where we let $w_k=0$ if the denominator is zero.
Then it holds $\|w_k\|\leq\rho/T$ and
\begin{align}
\sum_{m=(k-1)T+1}^{kT}\langle\nabla(F_S)_r(z_m),w_k\rangle &=-\frac{\rho}{T}\biggl\|\sum_{m=(k-1)T+1}^{kT}\nabla(F_S)_r(z_m)\biggr\|. \notag
\end{align}
For a given training set $S$, we let
\begin{align}
L_{S,p} &=\left(\frac{1}{n}\sum_{i=1}^nL(\xi_i)^p\right)^{\frac{1}{p}}. \notag
\end{align}

By Assumption~\ref{ass_lipschitz} and the definitions of $v_m$ and $(F_S)_r(\cdot)$, it holds
\begin{align}
\mathbb{E}_{i_m,u_m}\bigl[v_m\mid S,z_m\bigr] &=\nabla(F_S)_r(z_m), \label{eq:o2nc_unbiasedness}\\
\mathbb{E}_{i_m,u_m}\bigl[\|v_m\|^p\mid S,z_m\bigr] &\leq L_{S,p}^p. \label{eq:o2nc_gradient_moment}
\end{align}
By equation \eqref{eq:o2nc_gradient_moment}, we have
\begin{align}
\mathbb{E}_A\biggl[\sum_{m=(k-1)T+1}^{kT}\|v_m\|^p\,\biggm|\,S\biggr] &\leq \frac{T}{n}\sum_{i=1}^nL(\xi_i)^p=TL_{S,p}^p. \label{eq:o2nc_gradient_sum}
\end{align}
According to Lemma~\ref{lem:ogd_pathwise_regret} and equation  \eqref{eq:o2nc_gradient_sum}, we obtain
\begin{align}
\mathbb{E}_A\left[\sum_{m=(k-1)T+1}^{kT}\langle v_m,x_m-w_k\rangle\,\biggm|\,S\right] &\leq \frac{\rho^2}{2\eta T^2}+2\left(\frac{\rho}{T}\right)^{2-p}\eta^{p-1}TL_{S,p}^p \notag\\
&=\frac{\rho}{T^{\frac{p-1}{p}}}\left(\frac{L}{2}+\frac{2L_{S,p}^p}{L^{p-1}}\right), \label{eq:o2nc_regret_bound}
\end{align}
where the equality follows from $\eta=\rho/(LT^{1+1/p})$.

By equation \eqref{eq:o2nc_unbiasedness} and the definitions of $z_m$ and $y_m$, we have
\begin{align}
\mathbb{E}\bigl[\langle v_m,x_m\rangle\mid S,y_{m-1},x_m\bigr] &=\int_0^1\langle\nabla(F_S)_r(y_{m-1}+sx_m),x_m\rangle\,\mathrm{d}s \notag\\
&=(F_S)_r(y_m)-(F_S)_r(y_{m-1}). \notag
\end{align}
Let
$(F_S)_r^* = \inf_{x\in\mathbb{R}^d}(F_S)_r(x)$. 
Taking the conditional expectation for given $S$ and summing over $m=1,\ldots,KT$ yield
\begin{align}
\mathbb{E}_A\biggl[\sum_{m=1}^{KT}\langle v_m,x_m\rangle\,\biggm|\,S\biggr] &=\mathbb{E}_A\bigl[(F_S)_r(y_{KT})-(F_S)_r(y_0)\mid S\bigr] \geq-\bigl((F_S)_r(y_0)-(F_S)_r^*\bigr). \label{eq:o2nc_telescoping}
\end{align}

By Jensen's inequality and \eqref{eq:o2nc_unbiasedness}--\eqref{eq:o2nc_gradient_moment}, we have
\begin{align}
&\quad\mathbb{E}_{i_m,u_m}\bigl[\|v_m-\nabla(F_S)_r(z_m)\|^p\mid S,z_m\bigr] \notag\\
&\leq 2^{p-1}\left(\mathbb{E}_{i_m,u_m}\bigl[\|v_m\|^p\mid S,z_m\bigr]+\|\nabla(F_S)_r(z_m)\|^p\right) \notag\\
&\leq (2L_{S,p})^p. \label{eq:o2nc_noise_moment}
\end{align}
The construction of $i_m$ and $u_m$ in the algorithm ensures that equation \eqref{eq:o2nc_unbiasedness} also holds when conditioning on $S$, $\{i_\ell\}_{\ell=1}^{m-1}$, $\{u_\ell\}_{\ell=1}^{m-1}$, and $\{s_\ell\}_{\ell=1}^{m}$.
This implies for fixed $S$, the sequence~$\{v_m-\nabla(F_S)_r(z_m)\}_{m=1}^{KT}$ is a martingale difference sequence. 
Hence, we can apply Lemma~\ref{lem:mds_pmoment} and Jensen's inequality to achieve
\begin{align}
&\quad\mathbb{E}_A\Bigl[\Bigl\|\sum_{m=(k-1)T+1}^{kT}\bigl(v_m-\nabla(F_S)_r(z_m)\bigr)\Bigr\|\,\Bigm|\,S\Bigr] \notag\\
&\leq \Bigg(2\sum_{m=(k-1)T+1}^{kT}\mathbb{E}_A\bigl[\|v_m-\nabla(F_S)_r(z_m)\|^p\mid S\bigr]\Bigg)^{\frac{1}{p}} \notag\\
&\leq 4L_{S,p}T^{\frac{1}{p}}, \label{eq:o2nc_noise_sum}
\end{align}
where the last inequality is based on equation \eqref{eq:o2nc_noise_moment}.

By the definition of $w_k$, it holds
\begin{align}
\frac{\rho}{T}\biggl\|\sum_{m=(k-1)T+1}^{kT}\nabla(F_S)_r(z_m)\biggr\| &=\sum_{m=(k-1)T+1}^{kT}\langle v_m,x_m-w_k\rangle-\sum_{m=(k-1)T+1}^{kT}\langle v_m,x_m\rangle \notag\\
&~~\quad+\biggl\langle\sum_{m=(k-1)T+1}^{kT}\bigl(v_m-\nabla(F_S)_r(z_m)\bigr),w_k\biggr\rangle. \notag
\end{align}
Combining equations \eqref{eq:o2nc_regret_bound}, \eqref{eq:o2nc_telescoping}, and \eqref{eq:o2nc_noise_sum}, and using the fact $\|w_k\|\leq\rho/T$, we obtain
\begin{align}
\begin{split}
\label{eq:nonsmooth_sum_con_S}
&\quad\frac{1}{K}\sum_{k=1}^K\mathbb{E}_A\biggl[\biggl\|\frac{1}{T}\sum_{m=(k-1)T+1}^{kT}\nabla(F_S)_r(z_m)\biggr\|\,\biggm|\,S\biggr] \\
&\leq \frac{(F_S)_r(y_0)-(F_S)_r^*}{\rho K}+\frac{1}{T^{\frac{p-1}{p}}}\left(\frac{L}{2}+\frac{2L_{S,p}^p}{L^{p-1}}+4L_{S,p}\right).
\end{split}
\end{align}
Since it holds $(F_S)_r(x)=\mathbb{E}_u[F_S(x+ru)]\geq F_S^*$ for all $x$, we have $(F_S)_r^*\geq F_S^*$.
By Assumption~\ref{ass_lipschitz}, it also holds that 
\begin{align*}
(F_S)_r(y_0)\leq F_S(y_0)+\frac{r}{n}\sum_{i=1}^nL(\xi_i).    
\end{align*}
Therefore, by Jensen's inequality, we obtain
\begin{align}\label{eq:o2nc_initial_gap}
\begin{split}    
& \mathbb{E}_S[(F_S)_r(y_0)-(F_S)_r^*] \\
\leq & \mathbb{E}_S\left[F_S(y_0)-F_S^*+\frac{r}{n}\sum_{i=1}^nL(\xi_i)\right] \\ 
=& \Delta+r\mathbb{E}_{\xi\sim\mathcal D}[L(\xi)]\leq\Delta+rL, 
\end{split}
\end{align}
where $\Delta=\mathbb{E}_S[F_S(y_0)-\inf_{x\in\mathbb{R}^d}F_S(x)]$.
Furthermore, Assumption~\ref{ass_lipschitz} and Jensen's inequality give $\mathbb{E}_S[L_{S,p}^p]\leq L^p$ and $\mathbb{E}_S[L_{S,p}]\leq L$.
Since the constant $L$ is independent of $S$, taking expectation over $S$ on equation \eqref{eq:nonsmooth_sum_con_S} and using \eqref{eq:o2nc_initial_gap} yield
\begin{align}
\mathbb{E}_{S,A}\biggl[\biggl\|\frac{1}{T}\sum_{m=(R-1)T+1}^{RT}\nabla(F_S)_r(z_m)\biggr\|\biggr] &\leq \frac{\Delta+rL}{\rho K}+\frac{7L}{T^{\frac{p-1}{p}}}. \label{eq:o2nc_optimization_error}
\end{align}

By the triangle inequality and the uniform distribution of $R$, we have
\begin{align}
&\quad \mathbb{E}_{S,A}\biggl[\biggl\|\frac{1}{T}\sum_{m=(R-1)T+1}^{RT}\bigl(\nabla F_r(z_m)-\nabla(F_S)_r(z_m)\bigr)\biggr\|\biggr] \notag\\
&\leq \frac{1}{KT}\sum_{m=1}^{KT}\mathbb{E}_{S,A}\bigl[\|\nabla F_r(z_m)-\nabla(F_S)_r(z_m)\|\bigr] \notag\\
&\leq 2\epsilon_{\mathrm{O2NC}}+C_p(2L)^{2-p}\biggl(\frac{(2L)^p}{n}+4\epsilon_{\mathrm{O2NC}}^p\biggr)^{\frac{p-1}{p}}, \label{eq:o2nc_generalization_error}
\end{align}
where the last inequality follows from Proposition~\ref{prop:o2nc_stability} and Theorem~\ref{thm:nonsmooth_generalization}.

Finally, based on the definitions of $y_m$ and $z_m$ and the bound $\|x_m\|\leq\rho/T$, for all~$k\in[K]$ and $m\in\{(k-1)T+1,\ldots,kT\}$, it holds
\begin{align}
\|z_m-\bar z_k\| &\leq\frac{1}{T}\sum_{j=(k-1)T+1}^{kT}\|z_m-z_j\|\leq\rho. \notag
\end{align}
According to Lemma \ref{lem:smoothing_goldstein}, it holds $\nabla F_r(z_m)\in\partial_rF(z_m)\subseteq\partial_{r+\rho}F(\bar z_k)$.
The convexity of~$\partial_{r+\rho}F(\bar z_k)$ further gives
\begin{align}
\frac{1}{T}\sum_{m=(k-1)T+1}^{kT}\nabla F_r(z_m) &\in\partial_{r+\rho}F(\bar z_k). \notag
\end{align}
Since the output is $A(S)=\bar z_R$, the above result and the triangle inequality yield
\begin{align*}
\operatorname{dist}(0,\partial_{r+\rho}F(A(S))) &\leq \biggl\|\frac{1}{T}\sum_{m=(R-1)T+1}^{RT}\nabla(F_S)_r(z_m)\biggr\| \\
&\quad+\biggl\|\frac{1}{T}\sum_{m=(R-1)T+1}^{RT}\bigl(\nabla F_r(z_m)-\nabla(F_S)_r(z_m)\bigr)\biggr\|. 
\end{align*}
Taking expectation on the above inequality and applying equations \eqref{eq:o2nc_optimization_error} and \eqref{eq:o2nc_generalization_error}, we complete the proof.
\end{proof}

\begin{corollary}
\label{cor:o2nc_rate}
Following the assumptions of Theorem~\ref{thm:o2nc_population}, we take
$r = \delta/2$, $\rho=\delta/2$, $K\asymp n^{\frac{p-1}{3p-1}}$, and $T\asymp n^{\frac{p}{3p-1}}$. 
Then it holds
\begin{align}
\mathbb{E}_{S,A}\bigl[\operatorname{dist}(0,\partial_\delta F(A(S)))\bigr] &=O\left(\left(L+\frac{\Delta}{\delta}\right)n^{-\frac{p-1}{3p-1}}\right). \notag
\end{align}
\end{corollary}

\begin{remark}
In a recent work, \citet[Proposition 5.1]{kornowski2025improved} established a population Goldstein stationarity bound that depends on $\sqrt{d}$.
In contrast, our result in above corollary is dimension-independent.
\end{remark}

\section{Conclusion}
\label{sec:conclusion}

This paper studies stability and generalization for nonconvex optimization under heavy-tailed noise, including both the smooth and the nonsmooth settings.
We propose a truncation technique that decomposes stochastic gradient noise into a truncated component and a residual tail, which forms the basis of the generalization analyses under weak moment conditions.
For smooth optimization problems, we establish population gradient norm bounds for clipped SGD, normalized SGD, and their variants.
For nonsmooth optimization problems, we introduce on-average gradient stability and develop a population Goldstein stationarity bound for online-to-nonconvex conversion method.
For future work, it is interesting to study the high-probability guarantees and lower bounds for generalization error under heavy-tailed noise.

\appendix

\section{Generalization Bounds for Normalized SGD Variants}
\label{app:smooth_momentum}

This appendix presents the stability and population gradient bounds for two variants of normalized SGD, i.e., normalized SGD with momentum \citep{cutkosky2020momentum,hubler2025gradient} and normalized SGD with clipping and momentum \citep{liu2023breaking,sun2025revisiting,cutkosky2021high}.

\subsection{Normalized SGD with Momentum}

This subsection considers the normalized SGD with momentum (NSGD-M), which is presented in Algorithm~\ref{alg:nsgd_m} \citep{cutkosky2020momentum,hubler2025gradient}.
We first establish its stability as follows.

\begin{proposition}[Stability of NSGD-M]
\label{prop:nsgd_m_stability}
Suppose Assumption~\ref{ass_smooth} holds and the initial point is $x_0=0$.
Then Algorithm~\ref{alg:nsgd_m} is $\epsilon_{\mathrm{m}}$-uniformly gradient stable with
\begin{align*}
\epsilon_{\mathrm{m}} &\leq 2L\sqrt{\frac{T}{n}}\sum_{t=0}^{T-1}\eta_t.
\end{align*}
\end{proposition}

\begin{proof}
Without loss of generality, suppose that neighboring datasets $S$ and $\widetilde S$ differ only in their $n$th samples.
Let $I(A)=\{i_t\}_{t=0}^{T-1}$ denote the sampled index sequence, and let $R$ be the output index sampled uniformly from $\{0,\ldots,T-1\}$.
We run Algorithm~\ref{alg:nsgd_m} on $S$ and $\widetilde S$ using the same $I(A)$ and $R$.
Accordingly, we use $\{x_t\}_{t=0}^T$ and $\{\widetilde x_t\}_{t=0}^T$ to denote the corresponding iterates.
The outputs are given by $A(S)=x_R$ and~$A(\widetilde S)=\widetilde x_R$, respectively.

By the update rule and the convention $m_t/\|m_t\|=0$ when $m_t=0$, it holds
\begin{align}
\|x_{t+1}-x_t\| &\leq \eta_t \qquad \text{and} \qquad \|\widetilde x_{t+1}-\widetilde x_t\| \leq \eta_t. \notag
\end{align}
The initialization $\widetilde x_0=x_0=0$ implies that the outputs satisfy
\begin{align}
\|A(S)\| &\leq \sum_{t=0}^{T-1}\eta_t \qquad \text{and} \qquad \|A(\widetilde S)\| \leq \sum_{t=0}^{T-1}\eta_t. \label{eq:nsgd_m_output_bound}
\end{align}
On the event $\{n\notin I(A)\}$, the two sequences $\{x_t\}_{t=0}^T$ and $\{\widetilde x_t\}_{t=0}^T$ are identical, which implies that $A(S)=A(\widetilde S)$.
Therefore, equation \eqref{eq:nsgd_m_output_bound} implies
\begin{align}
\mathbb{E}_A\bigl[\|A(S)-A(\widetilde S)\|^2\bigr] &\leq 4\left(\sum_{t=0}^{T-1}\eta_t\right)^2\mathbb{P}\bigl(n\in I(A)\bigr) \leq 4\left(\sum_{t=0}^{T-1}\eta_t\right)^2\frac{T}{n}, \label{eq:nsgd_m_output_difference}
\end{align}
where the second inequality follows from the union bound.
By Assumption~\ref{ass_smooth} and equation \eqref{eq:nsgd_m_output_difference}, for every $\xi\in\Xi$, we have
\begin{align}
\mathbb{E}_A\bigl[\|\nabla f(A(S);\xi)-\nabla f(A(\widetilde S);\xi)\|^2\bigr] &\leq L^2\mathbb{E}_A\bigl[\|A(S)-A(\widetilde S)\|^2\bigr] \leq 4L^2\left(\sum_{t=0}^{T-1}\eta_t\right)^2\frac{T}{n}, \notag
\end{align}
which proves the result by Definition~\ref{def:gradient_stability}.
\end{proof}

\begin{algorithm}[t]
\caption{Normalized SGD with momentum (NSGD-M)}
\label{alg:nsgd_m}
\begin{algorithmic}[1]
\STATE \textbf{Input:} Initial point $x_0$, stepsizes $\{\eta_t\}_{t=0}^{T-1}$, momentum parameter $\beta\in[0,1)$, $m_{-1}=0$
\FOR{$t=0,1,\ldots,T-1$}
    \STATE Sample $i_t$ uniformly from $[n]$ 
    \STATE $m_t=\beta m_{t-1}+(1-\beta)\nabla f(x_t;\xi_{i_t})$ \\[0.1cm]
    \STATE $x_{t+1}=x_t-\eta_t\cdot\dfrac{m_t}{\|m_t\|}$
\ENDFOR
\STATE \textbf{Output:} $x_R$, where $R$ is uniformly sampled from $\{0, 1, \dots, T-1\}$
\end{algorithmic}
\end{algorithm}

Combining the stability bound with an optimization analysis gives the following population gradient bound.

\begin{theorem}
\label{thm:nsgd_m_population}
Consider Algorithm~\ref{alg:nsgd_m} with constant stepsize $\eta_t=\eta>0$, momentum parameter~$\beta\in[0,1)$, and the initial point~$x_0=0$.
Suppose Assumptions \ref{ass_smooth} and \ref{ass_pbcm} hold, and there exists $G>0$ such that $\mathbb{E}_{i_t}\bigl[\|\nabla f(x_t;\xi_{i_t})\|^p\bigr]\leq G^p$ for all $t$.
Then it holds
\begin{align*}
\mathbb{E}_{S,A}\bigl[\|\nabla F(A(S))\|\bigr]
&\leq \frac{\Delta}{\eta T}+\frac{L\eta}{2}
+8G(1-\beta)^{\frac{p-1}{p}}+\frac{4G}{T}+\frac{4G\beta}{(1-\beta)T}\\
&\quad+\frac{2L\eta\beta}{1-\beta}
+8L\eta T\sqrt{\frac{T}{n}}+C_p\sigma_p n^{-\frac{p-1}{p}},
\end{align*}
where $A(S)$ is the output of Algorithm~\ref{alg:nsgd_m} and
$C_p>0$ is the constant defined in Theorem~\ref{thm:smooth_generalization}.
\end{theorem}

\begin{proof}
Combining Proposition~\ref{prop:nsgd_m_stability} with equation \eqref{eq:smooth_population_framework}, we obtain
\begin{align}
\mathbb{E}_{S,A}\bigl[\|\nabla F(A(S))\|\bigr] &\leq \mathbb{E}_{S,A}\bigl[\|\nabla F_S(A(S))\|\bigr] +8L\eta T\sqrt{\frac{T}{n}} +C_p\sigma_p n^{-\frac{p-1}{p}}. \label{eq:nsgd_m_population_reduction}
\end{align}
It remains to bound the empirical optimization error $\mathbb{E}_{S,A}\bigl[\|\nabla F_S(A(S))\|\bigr]$.

Let $\mathcal{F}_t$ be the sigma-field generated by $S$ and $i_0,\ldots,i_{t-1}$.
By the $L$-smoothness of $F_S$, one has
\begin{align}
F_S(x_{t+1}) &\leq F_S(x_t)-\eta\left\langle\nabla F_S(x_t),\frac{m_t}{\|m_t\|}\right\rangle+\frac{L\eta^2}{2}. \label{eq:nsgd_m_descent}
\end{align}
For any $u,v\in\mathbb{R}^d$, the Cauchy--Schwarz and triangle inequalities imply
\begin{align}
\left\langle u,\frac{v}{\|v\|}\right\rangle &=\left\langle u-v,\frac{v}{\|v\|}\right\rangle+\|v\| \geq-\|u-v\|+\|v\| \geq\|u\|-2\|u-v\|. \label{eq:nsgd_m_direction_bound}
\end{align}
Applying equation \eqref{eq:nsgd_m_direction_bound} with $u=\nabla F_S(x_t)$ and $v=m_t$ and substituting the resulting bound into equation \eqref{eq:nsgd_m_descent}, we obtain
\begin{align}
F_S(x_{t+1}) &\leq F_S(x_t)-\eta\|\nabla F_S(x_t)\| +2\eta\|m_t-\nabla F_S(x_t)\| +\frac{L\eta^2}{2}. \label{eq:nsgd_m_descent_error}
\end{align}
Taking expectation in equation \eqref{eq:nsgd_m_descent_error}, summing over $t=0,\ldots,T-1$, and using $F_S(x_T)\geq F_S^*$, we obtain
\begin{align}
\frac{1}{T}\sum_{t=0}^{T-1}\mathbb{E}_{S,A}\bigl[\|\nabla F_S(x_t)\|\bigr] &\leq \frac{\Delta}{\eta T} +\frac{2}{T}\sum_{t=0}^{T-1}\mathbb{E}_{S,A}\bigl[\|m_t-\nabla F_S(x_t)\|\bigr] +\frac{L\eta}{2}. \label{eq:nsgd_m_optimization}
\end{align}

We let
\begin{align}
\delta_t &= m_t-\nabla F_S(x_t) \qquad \text{and} \qquad \zeta_t=\nabla f(x_t;\xi_{i_t})-\nabla F_S(x_t). \notag
\end{align}
By the uniform sampling of $i_t$, it holds $\mathbb{E}[\zeta_t\mid\mathcal{F}_t]=0$.
Moreover, we have
\begin{align}
\mathbb{E}\bigl[\|\zeta_t\|^p\mid\mathcal{F}_t\bigr] &\leq 2^{p-1}\left(\mathbb{E}\bigl[\|\nabla f(x_t;\xi_{i_t})\|^p\mid\mathcal{F}_t\bigr]+\|\nabla F_S(x_t)\|^p\right) \leq 2^pG^p, \label{eq:nsgd_m_noise_moment}
\end{align}
where the last inequality follows from Jensen's inequality and the moment assumption.

For $t\geq1$, the momentum update gives the following recursion for $\delta_t$:
\begin{align}
\delta_t &=\beta\delta_{t-1} +\beta\bigl(\nabla F_S(x_{t-1})-\nabla F_S(x_t)\bigr) +(1-\beta)\zeta_t. \label{eq:nsgd_m_delta_recursion}
\end{align}
Iterating equation \eqref{eq:nsgd_m_delta_recursion} yields
\begin{align}
\delta_t &=\beta^t\delta_0 +(1-\beta)\sum_{\tau=1}^t\beta^{t-\tau}\zeta_\tau +\beta\sum_{\tau=1}^t\beta^{t-\tau}\bigl(\nabla F_S(x_{\tau-1})-\nabla F_S(x_\tau)\bigr). \label{eq:nsgd_m_delta_representation}
\end{align}
For $t=0$, the initialization $m_{-1}=0$ implies that $\delta_0$ satisfies
\begin{align}
\delta_0 &= (1-\beta)\zeta_0-\beta\nabla F_S(x_0). \notag
\end{align}
Taking norms and expectations, we obtain
\begin{align}
\mathbb{E}_{S,A}\bigl[\|\delta_0\|\bigr] &\leq (1-\beta)\mathbb{E}_{S,A}\bigl[\|\zeta_0\|\bigr] +\beta\mathbb{E}_{S,A}\bigl[\|\nabla F_S(x_0)\|\bigr] \leq 2G, \label{eq:nsgd_m_initial_error}
\end{align}
where the last inequality follows from Jensen's inequality and the moment assumption.

For the last term in equation \eqref{eq:nsgd_m_delta_representation}, it holds
\begin{align}
\left\|\beta\sum_{\tau=1}^t\beta^{t-\tau}\bigl(\nabla F_S(x_{\tau-1})-\nabla F_S(x_\tau)\bigr)\right\| &\leq L\eta\beta\sum_{k=0}^{t-1}\beta^k \leq \frac{L\eta\beta}{1-\beta}, \label{eq:nsgd_m_drift_bound}
\end{align}
where we use the $L$-smoothness of $F_S$ and the fact $\|x_\tau-x_{\tau-1}\|\leq\eta$.
Since the sequence $\{\beta^{t-\tau}\zeta_\tau\}_{\tau=1}^t$ is a martingale difference sequence, we can use Jensen's inequality, Lemma~\ref{lem:mds_pmoment}, and equation \eqref{eq:nsgd_m_noise_moment} to achieve
\begin{align}
(1-\beta)\mathbb{E}_{S,A}\left[\left\|\sum_{\tau=1}^t\beta^{t-\tau}\zeta_\tau\right\|\right] &\leq (1-\beta)\left(2\sum_{\tau=1}^t\beta^{(t-\tau)p}\mathbb{E}_{S,A}\bigl[\|\zeta_\tau\|^p\bigr]\right)^{\frac{1}{p}} \notag\\
&\leq 4G(1-\beta)(1-\beta^p)^{-\frac{1}{p}} \leq 4G(1-\beta)^{\frac{p-1}{p}}, \label{eq:nsgd_m_martingale_bound}
\end{align}
where the last inequality uses $1-\beta^p\geq1-\beta$.

Taking norms and expectations on equation \eqref{eq:nsgd_m_delta_representation} and applying equations \eqref{eq:nsgd_m_initial_error}, \eqref{eq:nsgd_m_drift_bound}, and~\eqref{eq:nsgd_m_martingale_bound}, we obtain that for all $t\geq1$, it holds
\begin{align}
\mathbb{E}_{S,A}\bigl[\|\delta_t\|\bigr] &\leq 2G\beta^t+4G(1-\beta)^{\frac{p-1}{p}}+\frac{L\eta\beta}{1-\beta}. \label{eq:nsgd_m_delta_bound}
\end{align}
Summing equation \eqref{eq:nsgd_m_delta_bound} over $t=1,\ldots,T-1$ and combining it with equation \eqref{eq:nsgd_m_initial_error}, we obtain
\begin{align}
\sum_{t=0}^{T-1}\mathbb{E}_{S,A}\bigl[\|m_t-\nabla F_S(x_t)\|\bigr] &\leq 2G+\frac{2G\beta}{1-\beta} +4GT(1-\beta)^{\frac{p-1}{p}} +\frac{L\eta\beta T}{1-\beta}, \label{eq:nsgd_m_error_sum}
\end{align}
where we used $\sum_{t=1}^{T-1}\beta^t\leq\beta/(1-\beta)$.
Substituting equation \eqref{eq:nsgd_m_error_sum} into equation \eqref{eq:nsgd_m_optimization} and using the uniform sampling of the output, one has
\begin{align}
\mathbb{E}_{S,A}\bigl[\|\nabla F_S(A(S))\|\bigr] &\leq \frac{\Delta}{\eta T}+\frac{L\eta}{2} +8G(1-\beta)^{\frac{p-1}{p}} +\frac{4G}{T}+\frac{4G\beta}{(1-\beta)T} +\frac{2L\eta\beta}{1-\beta}. \label{eq:nsgd_m_empirical_bound}
\end{align}
Substituting equation \eqref{eq:nsgd_m_empirical_bound} into equation \eqref{eq:nsgd_m_population_reduction} completes the proof.
\end{proof}

\begin{corollary}
\label{cor:nsgd_m_rate}
Following the assumptions of Theorem~\ref{thm:nsgd_m_population}, we take
$T \asymp n^{\frac{3p-2}{7p-6}}$,
$\eta\asymp n^{-\frac{2p-1}{7p-6}}$,
and
$1-\beta\asymp n^{-\frac{p}{7p-6}}$.
Then it holds
\begin{align*}
\mathbb{E}_{S,A}\bigl[\|\nabla F(A(S))\|\bigr]
&=O\left(n^{-\frac{p-1}{7p-6}}\right).
\end{align*}
\end{corollary}

The upper bounds of the population gradient norm for NSGD-M and NSGD-B have the same order.
In the case of $p=2$, the population gradient bound in Corollary~\ref{cor:nsgd_m_rate} reduces to~$O(n^{-1/8})$, matching the result of \citet{sun2025nonconvex} under the bounded variance condition.

\subsection{Normalized SGD with Clipping and Momentum}

This subsection considers normalized SGD with clipping and momentum (NSGD-CM), which is presented in Algorithm \ref{alg:nsgd_cm} \citep{liu2023breaking,sun2025revisiting,cutkosky2021high}. 
We first establish its stability as follows.

\begin{proposition}[Stability of NSGD-CM]
\label{prop:nsgd_cm_stability}
Suppose that Assumption~\ref{ass_smooth} holds and the initial point is $x_0=0$.
Then Algorithm~\ref{alg:nsgd_cm} is $\epsilon_{\mathrm{cm}}$-uniformly gradient stable with
\begin{align*}
\epsilon_{\mathrm{cm}} &\leq 2L\sqrt{\frac{T}{n}}\sum_{t=0}^{T-1}\eta_t.
\end{align*}
\end{proposition}

\begin{algorithm}[t]
\caption{Normalized SGD with clipping and momentum (NSGD-CM)}
\label{alg:nsgd_cm}
\begin{algorithmic}[1]
\STATE \textbf{Input:} Initial point $x_0$, stepsizes $\{\eta_t\}_{t=0}^{T-1}$, momentum parameter $\beta\in[0,1)$, clipping threshold $\gamma$, $m_{-1}=0$
\FOR{$t=0,1,\ldots,T-1$}
    \STATE Sample $i_t$ uniformly from $[n]$
    \STATE $g_t=\mathrm{clip}_{\gamma}\bigl(\nabla f(x_t;\xi_{i_t})\bigr)$
    \STATE $m_t=\beta m_{t-1}+(1-\beta)g_t$
    \STATE $x_{t+1}=x_t-\eta_t\dfrac{m_t}{\|m_t\|}$
\ENDFOR
\STATE \textbf{Output:} $x_R$, where $R$ is uniformly sampled from $\{0, 1, \dots, T-1\}$
\end{algorithmic}
\end{algorithm}

\begin{proof}
Without loss of generality, suppose that neighboring datasets $S$ and $\widetilde S$ differ only in their $n$th samples.
Let $I(A)=\{i_t\}_{t=0}^{T-1}$ denote the sampled index sequence, and let $R$ be the output index sampled uniformly from $\{0,\ldots,T-1\}$.
We run Algorithm~\ref{alg:nsgd_cm} on $S$ and $\widetilde S$ using the same $I(A)$ and $R$.
Accordingly, we use $\{x_t\}_{t=0}^T$ and $\{\widetilde x_t\}_{t=0}^T$ to denote the corresponding iterates.
The outputs are given by $A(S)=x_R$ and~$A(\widetilde S)=\widetilde x_R$, respectively.

By the update rule and the convention $m_t/\|m_t\|=0$ when $m_t=0$, it holds
\begin{align}
\|x_{t+1}-x_t\| &\leq \eta_t \qquad \text{and} \qquad \|\widetilde x_{t+1}-\widetilde x_t\| \leq \eta_t. \notag
\end{align}
The initialization $\widetilde x_0=x_0=0$ implies that the outputs satisfy
\begin{align}
\|A(S)\| &\leq \sum_{t=0}^{T-1}\eta_t \qquad \text{and} \qquad \|A(\widetilde S)\| \leq \sum_{t=0}^{T-1}\eta_t. \label{eq:nsgd_cm_output_bound}
\end{align}
On the event $\{n\notin I(A)\}$, the two sequences $\{x_t\}_{t=0}^T$ and $\{\widetilde x_t\}_{t=0}^T$ are identical, which implies that $A(S)=A(\widetilde S)$.
Therefore, equation \eqref{eq:nsgd_cm_output_bound} implies
\begin{align}
\mathbb{E}_A\bigl[\|A(S)-A(\widetilde S)\|^2\bigr] &\leq 4\left(\sum_{t=0}^{T-1}\eta_t\right)^2\mathbb{P}\bigl(n\in I(A)\bigr) \leq 4\left(\sum_{t=0}^{T-1}\eta_t\right)^2\frac{T}{n}, \label{eq:nsgd_cm_output_difference}
\end{align}
where the second inequality follows from the union bound.
Assumption~\ref{ass_smooth} and equation~\eqref{eq:nsgd_cm_output_difference} then imply that for all $\xi\in\Xi$, it holds
\begin{align}
\mathbb{E}_A\bigl[\|\nabla f(A(S);\xi)-\nabla f(A(\widetilde S);\xi)\|^2\bigr] &\leq L^2\mathbb{E}_A\bigl[\|A(S)-A(\widetilde S)\|^2\bigr] \leq 4L^2\left(\sum_{t=0}^{T-1}\eta_t\right)^2\frac{T}{n}. \notag
\end{align}
According to Definition~\ref{def:gradient_stability}, we have finished the proof. 
\end{proof}

Combining the stability bound with an optimization analysis gives the following population gradient bound.

\begin{theorem}
\label{thm:nsgd_cm_population}
Consider Algorithm~\ref{alg:nsgd_cm} with constant stepsize $\eta_t=\eta>0$, momentum parameter~$\beta\in[0,1)$, clipping threshold~$\gamma>0$, and the initial point~$x_0=0$.
Suppose Assumptions~\ref{ass_smooth} and \ref{ass_pbcm} hold, and there exists $G>0$ such that $\mathbb{E}_{i_t}\bigl[\|\nabla f(x_t;\xi_{i_t})\|^p\bigr]\leq G^p$ for all $t$.
Then it~holds
\begin{align*}
\mathbb{E}_{S,A}\bigl[\|\nabla F(A(S))\|\bigr]
&\leq \frac{\Delta}{\eta T}+\frac{L\eta}{2}
+\frac{2L\eta}{1-\beta}+2G^p\gamma^{1-p}+8G(1-\beta)^{\frac{p-1}{p}}\\
&\quad+\frac{4G}{T}+\frac{4G\beta}{(1-\beta)T}
+8L\eta T\sqrt{\frac{T}{n}}+C_p\sigma_p n^{-\frac{p-1}{p}},
\end{align*}
where $A(S)$ is the output of Algorithm~\ref{alg:nsgd_cm} and
$C_p>0$ is the constant defined in Theorem~\ref{thm:smooth_generalization}.
\end{theorem}

\begin{proof}
Combining Proposition~\ref{prop:nsgd_cm_stability} with equation \eqref{eq:smooth_population_framework}, we obtain
\begin{align}
\mathbb{E}_{S,A}\bigl[\|\nabla F(A(S))\|\bigr] &\leq \mathbb{E}_{S,A}\bigl[\|\nabla F_S(A(S))\|\bigr] +8L\eta T\sqrt{\frac{T}{n}} +C_p\sigma_p n^{-\frac{p-1}{p}}. \label{eq:nsgd_cm_population_reduction}
\end{align}
It remains to bound the empirical optimization error $\mathbb{E}_{S,A}\bigl[\|\nabla F_S(A(S))\|\bigr]$.

Let $\mathcal{F}_t$ be the sigma-field generated by $S$ and $i_0,\ldots,i_{t-1}$.
By the $L$-smoothness of $F_S$, one has
\begin{align}
F_S(x_{t+1}) &\leq F_S(x_t)-\eta\left\langle\nabla F_S(x_t),\frac{m_t}{\|m_t\|}\right\rangle+\frac{L\eta^2}{2}. \label{eq:nsgd_cm_descent}
\end{align}
For any $u,v\in\mathbb{R}^d$, the Cauchy--Schwarz and triangle inequalities imply
\begin{align}
\left\langle u,\frac{v}{\|v\|}\right\rangle &=\left\langle u-v,\frac{v}{\|v\|}\right\rangle+\|v\| \geq-\|u-v\|+\|v\| \geq\|u\|-2\|u-v\|. \label{eq:nsgd_cm_direction_bound}
\end{align}
Applying equation \eqref{eq:nsgd_cm_direction_bound} with $u=\nabla F_S(x_t)$ and $v=m_t$ and substituting the resulting bound into equation \eqref{eq:nsgd_cm_descent}, we obtain
\begin{align}
F_S(x_{t+1}) &\leq F_S(x_t)-\eta\|\nabla F_S(x_t)\| +2\eta\|m_t-\nabla F_S(x_t)\| +\frac{L\eta^2}{2}. \label{eq:nsgd_cm_descent_error}
\end{align}
Taking expectation in equation \eqref{eq:nsgd_cm_descent_error}, summing over $t=0,\ldots,T-1$, and using $F_S(x_T)\geq F_S^*$, we obtain
\begin{align}
\frac{1}{T}\sum_{t=0}^{T-1}\mathbb{E}_{S,A}\bigl[\|\nabla F_S(x_t)\|\bigr] &\leq \frac{\Delta}{\eta T} +\frac{2}{T}\sum_{t=0}^{T-1}\mathbb{E}_{S,A}\bigl[\|m_t-\nabla F_S(x_t)\|\bigr] +\frac{L\eta}{2}. \label{eq:nsgd_cm_optimization}
\end{align}

Let $\delta_t = m_t-\nabla F_S(x_t)$.
For $t\geq1$, the momentum update gives the recursion for $\delta_t$ as
\begin{align}
\delta_t &=\beta\delta_{t-1} +\beta\bigl(\nabla F_S(x_{t-1})-\nabla F_S(x_t)\bigr) +(1-\beta)\bigl(g_t-\nabla F_S(x_t)\bigr), \label{eq:nsgd_cm_delta_recursion}
\end{align}
which implies 
\begin{align}
\delta_t =\beta^t\delta_0 +(1-\beta)\sum_{\tau=1}^t\beta^{t-\tau}\bigl(g_\tau-\nabla F_S(x_\tau)\bigr) +\beta\sum_{\tau=1}^t\beta^{t-\tau}\bigl(\nabla F_S(x_{\tau-1})-\nabla F_S(x_\tau)\bigr). \label{eq:nsgd_cm_delta_representation}
\end{align}
For $t=0$, the initialization $m_{-1}=0$ implies that $\delta_0$ satisfies
\begin{align}
\delta_0 &= (1-\beta)g_0-\nabla F_S(x_0). \notag
\end{align}
Using the fact $\|g_0\|\leq\|\nabla f(x_0;\xi_{i_0})\|$ and Jensen's inequality, we obtain
\begin{align}
\mathbb{E}_{S,A}\bigl[\|\delta_0\|\bigr] &\leq (1-\beta)\mathbb{E}_{S,A}\bigl[\|g_0\|\bigr] +\mathbb{E}_{S,A}\bigl[\|\nabla F_S(x_0)\|\bigr] \leq 2G. \label{eq:nsgd_cm_initial_error}
\end{align}

For the last term in equation \eqref{eq:nsgd_cm_delta_representation}, the $L$-smoothness of $F_S$ and $\|x_\tau-x_{\tau-1}\|\leq\eta$ imply
\begin{align}
\left\|\beta\sum_{\tau=1}^t\beta^{t-\tau}\bigl(\nabla F_S(x_{\tau-1})-\nabla F_S(x_\tau)\bigr)\right\| &\leq \frac{L\eta\beta}{1-\beta}. \label{eq:nsgd_cm_drift_bound}
\end{align}
For the second term in equation \eqref{eq:nsgd_cm_delta_representation}, we use the decomposition
\begin{align}
g_\tau-\nabla F_S(x_\tau) &=g_\tau-\mathbb{E}[g_\tau\mid\mathcal{F}_\tau] +\mathbb{E}[g_\tau\mid\mathcal{F}_\tau]-\nabla F_S(x_\tau). \label{eq:nsgd_cm_noise_decomposition}
\end{align}
Note that the sequence $\{g_\tau-\mathbb{E}[g_\tau\mid\mathcal{F}_\tau]\}_{\tau=1}^t$ is a martingale difference sequence.
Since clipping does not increase the gradient norm, Jensen's inequality and the moment assumption show that
\begin{align}
\mathbb{E}\bigl[\|g_\tau-\mathbb{E}[g_\tau\mid\mathcal{F}_\tau]\|^p\mid\mathcal{F}_\tau\bigr] &\leq 2^p\mathbb{E}\bigl[\|g_\tau\|^p\mid\mathcal{F}_\tau\bigr] \leq 2^pG^p. \label{eq:nsgd_cm_noise_moment}
\end{align}
Hence, we can apply Jensen's inequality, Lemma~\ref{lem:mds_pmoment}, and equation \eqref{eq:nsgd_cm_noise_moment} to achieve
\begin{align}
\mathbb{E}_{S,A}\left[\left\|\sum_{\tau=1}^t\beta^{t-\tau}\bigl(g_\tau-\mathbb{E}[g_\tau\mid\mathcal{F}_\tau]\bigr)\right\|\right] &\leq 4G(1-\beta^p)^{-\frac{1}{p}}. \label{eq:nsgd_cm_martingale_bound}
\end{align}

By the second inequality in equation \eqref{eq:clipping_properties} and the moment assumption, the clipping bias~satisfies
\begin{align}
\left\|\mathbb{E}[g_\tau\mid\mathcal{F}_\tau]-\nabla F_S(x_\tau)\right\| &=\left\|\mathbb{E}\bigl[\mathrm{clip}_{\gamma}(\nabla f(x_\tau;\xi_{i_\tau}))-\nabla f(x_\tau;\xi_{i_\tau})\mid\mathcal{F}_\tau\bigr]\right\| \notag\\
&\leq \gamma^{1-p}\mathbb{E}\bigl[\|\nabla f(x_\tau;\xi_{i_\tau})\|^p\mid\mathcal{F}_\tau\bigr] \leq G^p\gamma^{1-p}. \label{eq:nsgd_cm_clipping_bias}
\end{align}
Applying the triangle inequality to equation \eqref{eq:nsgd_cm_noise_decomposition} and using equations \eqref{eq:nsgd_cm_martingale_bound} and \eqref{eq:nsgd_cm_clipping_bias}, we~obtain
\begin{align}
(1-\beta)\mathbb{E}_{S,A}\left[\left\|\sum_{\tau=1}^t\beta^{t-\tau}\bigl(g_\tau-\nabla F_S(x_\tau)\bigr)\right\|\right] &\leq 4G(1-\beta)^{\frac{p-1}{p}}+G^p\gamma^{1-p}, \label{eq:nsgd_cm_noise_bound}
\end{align}
where we used $1-\beta^p\geq1-\beta$ and $(1-\beta)\sum_{\tau=1}^t\beta^{t-\tau}\leq1$.

Taking norms and expectations in equation \eqref{eq:nsgd_cm_delta_representation} and applying equations \eqref{eq:nsgd_cm_initial_error}, \eqref{eq:nsgd_cm_drift_bound}, and~\eqref{eq:nsgd_cm_noise_bound}, we obtain, for all $t\geq1$,
\begin{align}
\mathbb{E}_{S,A}\bigl[\|\delta_t\|\bigr] &\leq 2G\beta^t +4G(1-\beta)^{\frac{p-1}{p}} +G^p\gamma^{1-p} +\frac{L\eta\beta}{1-\beta}. \label{eq:nsgd_cm_delta_bound}
\end{align}
Summing equation \eqref{eq:nsgd_cm_delta_bound} over $t=1,\ldots,T-1$ and combining it with equation \eqref{eq:nsgd_cm_initial_error}, we~obtain
\begin{align}
\sum_{t=0}^{T-1}\mathbb{E}_{S,A}\bigl[\|m_t-\nabla F_S(x_t)\|\bigr] &\leq 2G+\frac{2G\beta}{1-\beta} +4GT(1-\beta)^{\frac{p-1}{p}} +G^p\gamma^{1-p}T +\frac{L\eta\beta T}{1-\beta}, \label{eq:nsgd_cm_error_sum}
\end{align}
where we used $\sum_{t=1}^{T-1}\beta^t\leq\beta/(1-\beta)$.
Substituting equation \eqref{eq:nsgd_cm_error_sum} into equation \eqref{eq:nsgd_cm_optimization} and using the uniform sampling of the output, one has 
\begin{align}\label{eq:nsgd_cm_empirical_bound}
\begin{split}    
& \mathbb{E}_{S,A}\bigl[\|\nabla F_S(A(S))\|\bigr]\\
\leq & \frac{\Delta}{\eta T}+\frac{L\eta}{2} +2G^p\gamma^{1-p} +8G(1-\beta)^{\frac{p-1}{p}} +\frac{4G}{T} +\frac{4G\beta}{(1-\beta)T} +\frac{2L\eta\beta}{1-\beta}. \end{split}
\end{align}
Substituting equation \eqref{eq:nsgd_cm_empirical_bound} into equation \eqref{eq:nsgd_cm_population_reduction} and using $\beta\leq1$ completes the proof.
\end{proof}

\begin{corollary}
\label{cor:nsgd_cm_rate}
Following the assumptions of Theorem~\ref{thm:nsgd_cm_population}, we take
$T \asymp n^{\frac{3p-2}{7p-6}}$,
$\eta\asymp n^{-\frac{2p-1}{7p-6}}$,
$1-\beta\asymp n^{-\frac{p}{7p-6}}$,
and 
$\gamma\asymp n^{\frac{1}{7p-6}}$.
Then it holds
\begin{align*}
\mathbb{E}_{S,A}\bigl[\|\nabla F(A(S))\|\bigr]
&=O\left(n^{-\frac{p-1}{7p-6}}\right).
\end{align*}
\end{corollary}

\section{Proof of Lemma~\ref{lem:ogd_pathwise_regret}}
\label{app:lemma_proof}
For all $k\in[K]$ and $m\in\{(k-1)T+1,\ldots,kT-1\}$,
the nonexpansiveness of the Euclidean projection implies
\begin{align*}
\|x_m-x_{m+1}\| 
= \|\Pi_{\mathbb B(0,\rho/T)}\left(x_m\right)-\Pi_{\mathbb B(0,\rho/T)}\left(x_m-\eta v_m\right)\| 
\leq \|\eta v_m\| =  \eta \|v_m\|.
\end{align*}
Additionally, the constraint $x_m,x_{m+1}\in\mathbb B(0,\rho/T)$ gives
\begin{align*}
\|x_m-x_{m+1}\| \leq \|x_m\| + \|x_{m+1}\| \leq \frac{2\rho}{T}.
\end{align*} 
Combining the above two inequalities results in
\begin{align}
\|x_m-x_{m+1}\| \leq \min\left\{\eta\|v_m\|,\frac{2\rho}{T}\right\}.\label{eq:ogd_regret_onestep}
\end{align}
Since the point $x_{m+1}$ minimizes $\frac12\|x-(x_m-\eta v_m)\|^2$ over $\mathbb B(0,\rho/T)$, for all $w_k\in\mathbb B(0,\rho/T)$, the first-order optimality condition gives
\begin{align}
\langle x_{m+1}-x_m+\eta v_m,w_k-x_{m+1}\rangle &\geq 0. \notag
\end{align}
Rearranging terms yields
\begin{align}
\eta\langle v_m,x_{m+1}-w_k\rangle &\leq \langle x_m-x_{m+1},x_{m+1}-w_k\rangle \notag\\
&=\frac{\|x_m-w_k\|^2-\|x_{m+1}-w_k\|^2-\|x_m-x_{m+1}\|^2}{2}, \notag
\end{align}
which leads to
\begin{align}\label{eq:ogd_regret_decomposition}
\begin{split}    
\langle v_m,x_m-w_k\rangle &\leq \frac{\|x_m-w_k\|^2-\|x_{m+1}-w_k\|^2}{2\eta}+\langle v_m,x_m-x_{m+1}\rangle-\frac{\|x_m-x_{m+1}\|^2}{2\eta} \\
&\leq \frac{\|x_m-w_k\|^2-\|x_{m+1}-w_k\|^2}{2\eta}+\|v_m\|\min\left\{\eta\|v_m\|,\frac{2\rho}{T}\right\} \\
&\leq \frac{\|x_m-w_k\|^2-\|x_{m+1}-w_k\|^2}{2\eta}+2\left(\frac{\rho}{T}\right)^{2-p}\eta^{p-1}\|v_m\|^p,
\end{split}
\end{align}
where the second inequality follows from the Cauchy--Schwarz inequality and equation \eqref{eq:ogd_regret_onestep}, and the last inequality uses $\min\{a,b\}\leq a^{p-1}b^{2-p}$ for $a\geq0$, $b>0$, and $p\in(1,2]$, together with $2^{2-p}\leq2$.

For the last iteration in each epoch, let $\widehat x_{kT+1}=\Pi_{\mathbb B(0,\rho/T)}(x_{kT}-\eta v_{kT})$.
Following the derivation of equation \eqref{eq:ogd_regret_decomposition} by replacing $(x_m,x_{m+1},v_m)$ with $(x_{kT},\widehat x_{kT+1},v_{kT})$ yields
\begin{align}
\langle v_{kT},x_{kT}-w_k\rangle &\leq \frac{\|x_{kT}-w_k\|^2-\|\widehat x_{kT+1}-w_k\|^2}{2\eta}+2\left(\frac{\rho}{T}\right)^{2-p}\eta^{p-1}\|v_{kT}\|^p. \label{eq:ogd_regret_laststep}
\end{align}
Summing inequality \eqref{eq:ogd_regret_decomposition} over $m=(k-1)T+1,\ldots,kT-1$ and adding \eqref{eq:ogd_regret_laststep}, we have
\begin{align}
\sum_{m=(k-1)T+1}^{kT}\langle v_m,x_m-w_k\rangle &\leq \frac{\|x_{(k-1)T+1}-w_k\|^2-\|\widehat x_{kT+1}-w_k\|^2}{2\eta} \notag\\
&\quad+2\left(\frac{\rho}{T}\right)^{2-p}\eta^{p-1}\sum_{m=(k-1)T+1}^{kT}\|v_m\|^p \notag\\
&\leq \frac{\rho^2}{2\eta T^2}+2\left(\frac{\rho}{T}\right)^{2-p}\eta^{p-1}\sum_{m=(k-1)T+1}^{kT}\|v_m\|^p, \notag
\end{align}
where the last inequality follows from $x_{(k-1)T+1}=0$ and $\|w_k\|\leq\rho/T$.

\bibliographystyle{plainnat}
\bibliography{arxiv}
\end{document}